\documentclass[manuscript, 12pt]{article}
\usepackage[margin=0.65in]{geometry}
\usepackage{graphicx}
\usepackage{enumerate}
\usepackage{amsmath,amsfonts,amssymb,amsthm}
\usepackage{color}
\usepackage[ruled,vlined,linesnumbered]{algorithm2e}
\usepackage{hyperref}
\usepackage[round]{natbib}

\usepackage{pgf, tikz, subcaption}
\usetikzlibrary{arrows, automata,calc,shapes}
\usetikzlibrary{decorations.markings}
\tikzstyle{c1}=[circle,thick,auto,draw,inner
        sep=2pt,minimum size=1.5em,fill=lightgray!10]
        
\newtheorem{thm}{Theorem}[section]
\newtheorem{cor}[thm]{Corollary}
\newtheorem{prop}[thm]{Proposition}
\newtheorem{lem}[thm]{Lemma}

\newtheorem{defn}[thm]{Definition}

\newtheorem{rem}[thm]{Remark}

\newtheorem{assm}[thm]{Assumption}

\newcommand{\PA}{\textnormal{pa}}
\newcommand{\ADJ}{\textnormal{adj}}
\newcommand{\CH}{\textnormal{ch}}
\newcommand{\DAG}{{DAG}}
\newcommand{\DAGS}{{DAGs}}
\newcommand{\amat}{{\Theta}}
\newcommand{\pa}{\textnormal{pa}}
\newcommand{\ra}{\rightarrow}
\newcommand{\nra}{\nrightarrow}
\newcommand{\la}{\leftarrow}

\newcommand{\indep}{{\bot\negthickspace\negthickspace\bot}}
\newcommand{\dep}{{\medspace\slash\negmedspace\negmedspace\negthickspace\bot\negthickspace\negthickspace\bot}}
\newcommand{\RR}{{\mathbb{R}}}
\newcommand{\cmb}{{\mathrm{cmb}}}
\newcommand{\X}{{{\mathcal{X}}}}
\newcommand{\Y}{{{\mathcal{Y}}}}
\newcommand{\W}{{{\mathcal{W}}}}

\newcommand{\PP}[1]{\mathbb{P}\left\{{#1}\right\}} 
\newcommand{\EE}[1]{\mathbb{E}\left[{#1}\right]} 
\newcommand\independent{\protect\mathpalette{\protect\independenT}{\perp}}
\def\independenT#1#2{\mathrel{\rlap{$#1#2$}\mkern2mu{#1#2}}}

\DeclareMathOperator*{\argmin}{arg\,min}
\newcommand{\norm}[1]{\lVert{#1}\rVert}

\usepackage{xargs}
\usepackage[colorinlistoftodos,prependcaption,textsize=tiny]{todonotes}
\newcommandx{\unsure}[2][1=]{\todo[linecolor=red,backgroundcolor=red!25,bordercolor=red,#1]{#2}}
\newcommandx{\change}[2][1=]{\todo[linecolor=blue,backgroundcolor=blue!25,bordercolor=blue,#1]{#2}}
\newcommandx{\info}[2][1=]{\todo[linecolor=OliveGreen,backgroundcolor=OliveGreen!25,bordercolor=OliveGreen,#1]{#2}}
\newcommandx{\improvement}[2][1=]{\todo[linecolor=Plum,backgroundcolor=Plum!25,bordercolor=Plum,#1]{#2}}

\title{Learning Directed Acyclic Graphs From Partial Orderings}
\author{Ali Shojaie$^{\ast,\dagger}$, Wenyu Chen$^\ast$ \\ Department of Biostatistics, University of Washington}
\date{}

\begin{document}

\maketitle

\def\thefootnote{$\ast$}\footnotetext{The authors contributed equally to this work.}\def\thefootnote{\arabic{footnote}}

\def\thefootnote{$\dagger$}\footnotetext{Corresponding eamil: ashojaie@uw.edu}\def\thefootnote{\arabic{footnote}}

\abstract{
Directed acyclic graphs ({\DAGS}) are commonly used to model causal relationships among random variables. In general, learning the {\DAG} structure is both computationally and statistically challenging. Moreover, without additional information, the direction of edges may not be estimable from observational data. In contrast, given a complete causal ordering of the variables, the problem can be solved efficiently, even in high dimensions. In this paper, we consider the intermediate problem of learning {\DAGS} when a partial causal ordering of variables is available. We propose a general estimation framework for leveraging the partial ordering and present efficient estimation algorithms for low- and high-dimensional problems. The advantages of the proposed framework are illustrated via numerical studies. 
}


\section{Introduction}\label{sec:intro}
Directed acyclic graphs ({\DAGS}) are widely used to capture causal relationships among components of complex systems \citep{spirtes_causation_2001,pearl_causality_2009,maathuis_handbook_2018}. They also form a foundation for causal discovery and inference \citep{pearl_causality_2009}. Probabilistic graphical models defined on {\DAGS}, known as Bayesian networks \citep{pearl_causality_2009}, have thus found broad applications in various scientific disciplines, from biology \citep{markowetz2007inferring, zhang2013integrated} and social sciences \citep{gupta2008linking}, to knowledge representation and machine learning \citep{heckerman1997bayesian}. However, learning the structure of {\DAGS} from observational data is very challenging due to at least two major factors: First, it may not be possible to infer the direction of edges from observational data alone. In fact, unless the model is \emph{identifiable} \citep[see, e.g.,][]{peters2014causal}, observational data only reveal the structure of the Markov equivalent class of {\DAGS} \citep{maathuis_handbook_2018}, captured by a complete partially directed acyclic graph (CPDAG) \citep{andersson_characterization_1997}. The second reason is computational---learning {\DAGS} from observational data is an NP-complete problem \citep{chickering_learning_1996}. In fact, while a few polynomial time algorithms have been proposed for special cases, including sparse graphs \citep{kalisch_estimating_2007} or identifiable models \citep{chen_causal_2019,ghoshal_learning_2018,peters_causal_2014,wang_high-dimensional_2020,shimizu_linear_2006,yu2020directed}, existing general-purpose algorithms are not scalable to problems involving many variables. 

\begin{figure}[t]
 \begin{center}
    \scalebox{0.75}{
      \begin{tikzpicture}[
		> = stealth, shorten > = 1pt,
		auto,node distance = 1.5cm,
		semithick, line width= 1.5,
		]
	\node[c1] at (0,0) (1) {$1$};
	\node[c1] at (2,0.9) (2) {$2$};
	\node[c1] at (2,-0.9) (3) {$3$};
	\node[c1] at (4,0) (4) {$4$};
        \draw[->] (1) -- (2);
        \draw[->] (1) -- (3);
        \draw[->] (2) -- (4);
        \draw[->] (3) -- (4);
      \end{tikzpicture}
    }
  \end{center}
  \caption{\footnotesize A directed graph with four nodes.\label{fig:simplegraph}}
\end{figure}
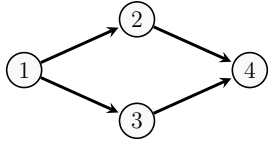

In spite of the many challenges of learning {\DAGS} in general settings, the problem becomes very manageable if a \emph{valid causal ordering} among variables is known \citep{shojaie_penalized_2010}. In a valid causal ordering for a {\DAG} $G$ with node set $V$, any node $j$ can appear before another node $k$ (denoted $j \prec k$) only if there is no directed path from $k$ to $j$. Multiple valid causal orderings may exists for a given {\DAG}, as illustrated in the simple example of Figure~\ref{fig:simplegraph}, where both $\mathcal{O}_1 = \{1 \prec 2 \prec 3 \prec 4\}$ and $\mathcal{O}_2 = \{1 \prec 3 \prec 2 \prec 4\}$ are valid causal orderings.

Clearly, a known causal ordering of variables resolves any ambiguity about the directions of edges in a {\DAG} and hence addresses the first source of difficulty in estimation of {\DAGS} discussed above. However, this knowledge also significantly simplifies the computation: given a valid causal ordering, {\DAG} learning reduces to a variable selection problem that can be solved efficiently even in the high-dimensional setting \citep{shojaie_penalized_2010}, when the number of variables is much larger than the sample size.

In the simplest case, the idea of \citet{shojaie_penalized_2010} is to regress each variable $k$ on all preceding variables in the ordering, $\{j: j \prec k\}$. While simple and efficient, this idea, and its extensions \citep[e.g.,][]{fu2013learning, shojaie_inferring_2014, han2016estimation}, require a \emph{complete} (or full) casual ordering of variables, i.e., a permutation of the list of variables in {\DAG} $G$. 
However, complete causal orderings are rarely available in practice. To relax this assumption, a few recent proposals have combined regularization strategies with algorithms that search over the space of orderings \citep[e.g.,][]{raskutti_learning_2018}. These algorithms are more efficient than those searching over the super-exponentially large space of {\DAGS} \citep{friedman2003being}. Nonetheless, the computation for these algorithms remains prohibitive for moderate to large size problems \citep{manzour_integer_2021, kucukyavuz_consistent_2022}. 

In this paper, we consider the setting where a \emph{partial} causal ordering of variables is known.
This scenario---which is an intermediate between assuming a complete causal ordering as in \citet{shojaie_penalized_2010} and no assumption on causal ordering---occurs commonly in practice. An important example is the problem of identifying \emph{direct} causal effects of multiple exposures on multiple outcomes (assuming no unmeasured confounders). Formally, let $\X = \{X_{1}, \ldots, X_{p}\}$ and $\Y = \{Y_{1}, \ldots, Y_{q}\}$ denote the set of $p$ exposure and $q$ outcome variables, respectively. 
Then, we have a partial ordering among $\X$ and $\Y$ variables, namely, $\X \prec \Y$, but we do not have any knowledge of the ordering among $\X$ or $\Y$ variables themselves. Nonetheless, we are interested in identifying direct causal effects of exposures $\X$ on outcomes $\Y$. This corresponds to learning edges from $\X$ to $\Y$, which would form a bipartite graph. 

Estimation of {\DAGS} from partial orderings also arises naturally in the analysis of biological systems. For instance, in gene regulatory networks, `transcription factors' are often known \emph{a priori} and they are not expected to be affected by other `target genes'. Similar to the previous example, here the set of transcription factors, $\X$, appear before the set of target genes, $\Y$, and the goal is to infer gene regulatory interactions. Similar problems also occur in integrative genomics, including in eQTL mapping \citep{ha_estimation_2020}. 

Despite its importance and many applications, the problem of learning {\DAGS} from partial orderings has not been satisfactorily addressed. In particular, as we show in the next section, various regression-based strategies currently used in applications result in incorrect estimates. As an alternative to these methods, one can use general {\DAG} learning algorithms, such as the PC algorithm \citep{spirtes_causation_2001,kalisch_estimating_2007}, to learn the structure of the CPDAG and then orient the edges between $\X$ and $\Y$  according to the known partial ordering. However, such an approach would not utilize the partial ordering in the estimation of edges and is thus inefficient. The recent proposal of \citet{wang2019directed} is also not computationally feasible as it requires searching over all possible orderings of variables. 

To overcome the above limitatiosn, we present a new framework for leveraging the partial ordering information into {\DAG} learning. To this end, after formulating the problem, in Section~\ref{sec:problem}, we also  investigate the limitations of existing approaches. Motivated by these findings, we present a new framework in Section~\ref{sec:method} and establish the correctness of its population version. Then, in Section~\ref{sec:estimation}, we present different estimation strategies for learning low- and high-dimensional {\DAGS}. 
To simplify the presentation, the main ideas are presented for the special case of two-layer networks corresponding to linear structural equation models (SEMs); the more general version of the algorithm and its extensions are discussed in Section~\ref{sec:extensions} and Appendix~\ref{sec:CAM}. 
The advantages of the proposed framework are illustrated through simulation studies and an application in integrative genomics in Section~\ref{sec:perfanal}.

\section{Learning Directed Graphs from Partial Orderings}\label{sec:problem}

\subsection{Problem Formulation}\label{sec:prob}
Consider a {\DAG} $G = (V, E)$ with the node set $V$, and the edge set $E \subset V \times V$. For the general problem, we assume that $V$ is partitioned into $L$ sets, $V_1, \ldots, V_L$ such that for any $\ell \in \{1, \ldots, L\}$, the nodes in $V_{\ell}$ cannot be parents of nodes in any set $V_{\ell'}, \ell' < \ell$. Such a partition defines a \emph{layering} of $G$ \citep{manzour_integer_2021}, denoted $V_1 \prec V_2 \prec \cdots \prec V_L$. 
In fact, a valid layering can be found for any {\DAG}, though some layers may contain a single node. 
As such, the notion of layering is general: We make no assumption on the size of each layer, the causal ordering of variables in each set $V_\ell$, or interactions among them, except that $G$ is a {\DAG}.

To simplify the presentation, we primarily focus on two-layer, or bipartite, {\DAGS}, and defer the discussion of more general cases to Section~\ref{sec:extensions}.
Let $V_1 = \X \equiv \{X_1, \ldots, X_p\}$ be the nodes in the first layer and $V_2 = \Y \equiv \{Y_1, \ldots, Y_q\}$ be those in the second layer. 
In the case of causal inference for multiple outcomes discussed in Section~\ref{sec:intro}, $\X$ represents the \emph{exposure} variables, and $\Y$ represents the \emph{outcome} variables.

Throughout the paper, we denote individual variables by regular upper case letters, e.g., $X_{k}$ and $Y_{j}$, and sets of variables as calligraphic upper case letters, e.g., $\X_{-k}$ and $\X_{S}$, where $S$ is a set containing more than one index. We also refer to nodes of the network by using both their indices and the corresponding random variables; for instance, $k \in \X$ and $X_k$. We denote the parents of a node $j$ in $G$ by $\PA_j$, its ancestors and children by $\mathrm{an}_j$ and $\CH_j$, respectively, and its adjacent nodes by $\ADJ_j$, where $\ADJ_j = \PA_j \cup \CH_j$. 
We may also represent the edges of $G$ using its adjacency matrix $\amat$, which satisfies $\amat_{jj} = 0$ and $\amat_{jj'} \neq 0$ if and only if $j' \in \pa_{j}$. For any bipartite {\DAG} with layers $\X$ and $\Y$ ($\X \prec \Y$), we can partition $\amat$ into the following block matrix
\begin{equation}\label{eq:adjmat}
	\amat = \left[
	\begin{array}{cc}
	  A   &   0 \\
	  B   &  C 
	\end{array}
	\right],
\end{equation}
where the zero constraint on the upper right block of $\amat$ follows from the partial ordering.
Here $A$ and $C$ contain the information on the edges amongst $X_k$  $(k = 1, \ldots, p)$ and $Y_j$ $(j = 1, \ldots, q)$, respectively. 
Both of these matrices can be written as lower-triangular matrices \citep{shojaie_penalized_2010}. 
However, there are generally no constraints on $B$.
We denote by $H$ the subgraph of $G$ containing edges from $\X$ to $\Y$ only, i.e., those corresponding to entries in $B$.
Our goal is to estimate $H$ using the fact that $\X \prec \Y$.

\begin{figure}[t]
\centering
\begin{subfigure}[b]{.25\linewidth}
    \centering
    \begin{tikzpicture}[
		> = stealth, shorten > = 1pt,
		auto,node distance = 1.5cm,
		semithick, line width= 1.5,
		]
		\node[c1] (x1) at(0,2) {$X_1$};
		\node[c1] (x2) at(2,2) {$X_2$};
		\node[c1] (y1) at(0,0) {$Y_1$};
		\node[c1] (y2) at(2,0) {$Y_2$};
		\draw[->,color = gray] (x1) -- (x2);
		\draw[->,color = blue] (x1) -- (y1);
		\draw[->,color = gray] (y1) -- (y2);
		\draw[->,color = blue] (x2) -- (y2);
		\end{tikzpicture}
		\caption{full \DAG}
	\end{subfigure}%
	\begin{subfigure}[b]{.25\linewidth}
	\centering
	\begin{tikzpicture}[
		> = stealth, shorten > = 1pt,
		auto,node distance = 1.5cm,
		semithick, line width= 1.5,
		]
		\node[c1] (x1) at(0,2) {$X_1$};
		\node[c1] (x2) at(2,2) {$X_2$};
		\node[c1] (y1) at(0,0) {$Y_1$};
		\node[c1] (y2) at(2,0) {$Y_2$};
		\draw[->,color = blue] (x1) -- (y1);
		\draw[->,color = blue] (x2) -- (y2);
		\draw[->,color = orange,dashed] (x1) -- (y2);
		\end{tikzpicture}
		\caption{$H^{(0)}$}
	\end{subfigure}%
	\begin{subfigure}[b]{.25\linewidth}
	\centering
	\begin{tikzpicture}[
		> = stealth, shorten > = 1pt,
		auto,node distance = 1.5cm,
		semithick, line width= 1.5,
		]
		\node[c1] (x1) at(0,2) {$X_1$};
		\node[c1] (x2) at(2,2) {$X_2$};
		\node[c1] (y1) at(0,0) {$Y_1$};
		\node[c1] (y2) at(2,0) {$Y_2$};
		\draw[->,color = blue] (x1) -- (y1);
		\draw[->,color = blue] (x2) -- (y2);
		\draw[->,color = orange,dashed] (x2) -- (y1);
		\end{tikzpicture}
		\caption{$H^{(-j)}$}
	\end{subfigure}%
	\begin{subfigure}[b]{.25\linewidth}
	\centering
	\begin{tikzpicture}[
		> = stealth, shorten > = 1pt,
		auto,node distance = 1.5cm,
		semithick, line width= 1.5,
		]
		\node[c1] (x1) at(0,2) {$X_1$};
		\node[c1] (x2) at(2,2) {$X_2$};
		\node[c1] (y1) at(0,0) {$Y_1$};
		\node[c1] (y2) at(2,0) {$Y_2$};
		\draw[->,color = blue] (x1) -- (y1);
		\draw[->,color = blue] (x2) -- (y2);
		\end{tikzpicture}
		\caption{$H$}
	\end{subfigure}%
\caption{Toy example illustrating estimation of {\DAGS} from partial orderings in a two-layer network with $\X \prec \Y$, $\X=\{X_1,X_2\}$ and $\Y = \{Y_1, Y_2\}$. a) The full {\DAG} $G$ with the edges between layers drawn in blue and edges within each layer shown in gray. Here, the true causal relations are linear and the goal is to estimate the bipartite graph $H$ defined by edges $X_1 \ra Y_1$ and $X_2 \ra Y_2$; b) Estimate of $H$ using $H^{(0)}$ in \eqref{eq:H0}, obtained by regressing each $Y_j$, $j = 1, 2$ on $\{X_1, X_2\}$ using a linear model with $n=1,000$ observations, and drawing an edge if the corresponding coefficient is significant at $\alpha = 0.05$; the graph contains a false positive edge shown by an orange dashed arrow; c) estimate of $H$ using $H^{(-j)}$ in \eqref{eq:Hminusj} obtained by regressing each $Y_j$, $j = 1, 2$ on $\{X_1, X_2, Y_{-j}\}$ using a linear regression similar to (b); d) estimate of $H$ obtained using the proposed approach.}\label{fig:toy}
\end{figure}
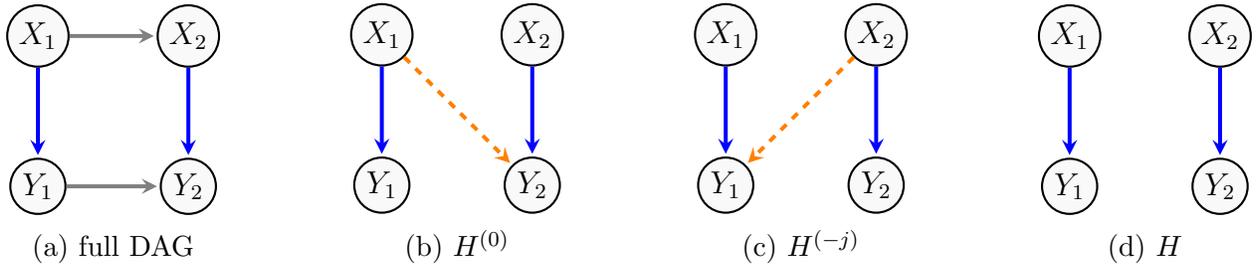

Figure~\ref{fig:toy}a shows a simple example of a two-layer {\DAG} $G$ with layers consisting of $p = q =2$ nodes. 
This example illustrates the difference between full causal orderings of variables in a {\DAG} 
and partial causal orderings: In this case, the graph admits two full causal orderings, namely $\mathcal{O}_1 = \{X_1,X_2,Y_1,Y_2\}$ and $\mathcal{O}_2 = \{X_1,Y_1,X_2,Y_2\}$; either of these orderings can be used to correctly discover the structure of the graph. Here, the partial ordering of the variables defined by the layering of the graph to sets $\X$ and $\Y$, i.e., $\X \prec \Y$, can be written as $\mathcal{O}' = \{ \{X_1,X_2\}, \{Y_1,Y_2\} \}$. This ordering determines that $X_k$s should appear before $Y_{j}$s in the causal ordering but does not restrict the ordering of $\{X_1,X_2\}$ and $\{Y_1,Y_2\}$.
In this example, only $\mathcal{O}_1$ is consistent with the partial ordering $\mathcal{O}'$. Moreover, $\mathcal{O}'$ is also consistent with $\{X_2,X_1,Y_2,Y_1\}$, which is not a valid causal ordering of $G$, indicating that a partial causal order is not sufficient for learning the full \DAG.

\subsection{Failure of Simple Algorithms}
\label{sec:challenge}
In this section, we discuss the challenges of estimating the graph $H$ and why this problem cannot be solved using simple approaches. 
Since the partial ordering of nodes, $V_1 \prec V_2$, provides information about direction of causality between the two layers of the network, we focus on simple \emph{constraint-based} methods \citep{spirtes_causation_2001}, which learn the network edges based on conditional independence relationships among nodes. 
This requires conditional independence relations among variables to be compatible with the edges in $G$; formally, the joint probability distribution $\mathcal{P}$ needs to be \emph{faithful} to $G$ \citep{spirtes_causation_2001}. (As discussed in Section~\ref{sec:estimation}, our algorithm requires a weaker notion of faithfulness; however, for simplicity, we consider the classical notion of faithfulness in this section.)

Given the partial ordering of variables---which means that $Y_j$s cannot be parents of $X_k$s---one approach for estimating the bipartite graph $H$ is to draw an edge from $X_k$ to $Y_j$ whenever $Y_j$ is dependent on $X_k$ given all other nodes in the first layer, $\X_{-k} \equiv \{X_{k'}, k' \ne k\}$. Formally, denoting by $Y_j \indep X_k$ the conditional independence of two variables $Y_j$ and $X_k$, we define
\begin{equation}\label{eq:H0}
H^{(0)} \equiv \left\{ (k \ra j): Y_j \dep X_k \mid \X_{-k} \right\}, 
\end{equation}
to emphasize that the estimate is obtained without conditioning on any $Y_{j'} \ne Y_j$. 

Figure~\ref{fig:toy}b shows the estimated $H^{(0)}$ in the setting where true causal relationships are linear. 
In this setting, edges in $H^{(0)}$ represent nonzero coefficients in linear regressions of each $Y_j$ on its parents in $G$, $\pa_j$. 
It can be seen that, in this example, the true causal effects from $X_k$s to $Y_j$s are in fact captured in $H^{(0)}$; these are shown in solid blue lines in Figure~\ref{fig:toy}b. 
However, this simple example suggests that $H^{(0)}$ may include spurious edges, shown by orange dashed edges in the figure. The next lemma formalizes and generalizes this finding. The proof of this and other results in the paper are gathered in the Appendix~\ref{sec:proofs}. 
\begin{lem}\label{lemma:1}
Assume that $\mathcal{P}$ is faithful 
with respect to $G$. Let $H^{(0)}$ be the directed bipartite graph defined in \eqref{eq:H0}. Then, if $\{ X_1, \ldots X_p \} \prec \{ Y_1, \ldots Y_q \}$,
\begin{enumerate}
\item[i)] $X_k \ra Y_j \in H^{(0)}$ whenever $X_k \ra Y_j \in G$;
\item[ii)] for any path of the form $X_{k_0} \ra Y_{j_1} \ra \cdots \ra Y_{j_0}$ such that $X_{k_0} \ra Y_{j_0} \notin G$, $H^{(0)}$ will include a \emph{false positive} edge $X_{k_0} \ra Y_{j_0}$.
\end{enumerate}
\end{lem}

Lemma~\ref{lemma:1} suggests that if $G$ includes edges among $Y_j$s, failing to condition on $Y_j$s can result in falsely detected causal effects from $X_k$s to $Y_j$s. Thus, without knowledge of interactions among $Y_j$s, one may consider an estimator that corrects for both $X_{- k}$ and $Y_{- j}$ when trying to detect the casual effects of $X_k$ on $Y_j$; in other words, we may declare an edge from $X_k$ to $Y_j$ whenever $Y_j \dep X_k \mid \{ \X_{- k}, \Y_{- j}\}$. 
We denote the resulting estimate of $H$ as $H^{(-j)}$:
\begin{equation}\label{eq:Hminusj}
H^{(-j)} = \{ (k \to j): Y_j \dep X_k \mid \X_{-k} \cup \Y_{-j} \}.
\end{equation}

Unfortunately, as Figure~\ref{fig:toy}c shows, estimation based on this model may also include false positive edges. 
As the next lemma clarifies, the false positive edges in this case are due to the conditioning on common descendants of a pair of $X_k$ and $Y_j$ that are not connected in $G$, which is sometimes referred to as Berkson's Paradox \citep{pearl_causality_2009}.  
\begin{lem}\label{lemma:2}
Assume that $\mathcal{P}$ is faithful with respect to $G$, and let $H^{(-j)}$ be the directed bipartite graph $H^{(-j)}$  defined in \eqref{eq:Hminusj}. 
Then, if $\{ X_1, \ldots, X_q \} \prec \{ Y_1, \ldots Y_p \}$, for any $X_{k} \ra Y_{j} \in G$, $X_{k} \ra Y_{j} \in H^{(-j)}$.
Moreover, for any triplets of nodes $X_{k_0}$, $Y_{j_0}$ and $Y_{j_1}$ that form an open collider in $G$ \citep{pearl_causality_2009}, i.e.,
\begin{itemize}
\item[-] $X_{k_0} \ra Y_{j_1} \la Y_{j_0}$
\item[-] $X_{k_0} \nra Y_{j_0}$
\end{itemize}
$H^{(-j)}$ will include a \emph{false positive} edge from $X_{k_0}$ to $Y_{j_0}$.
\end{lem}
\begin{rem}
Examining the proofs of Lemmas~\ref{lemma:1} and \ref{lemma:2}, if $C = 0$ in \eqref{eq:adjmat}, then $H^{(0)} = H^{(-j)} = H$. In other words, if $G$ does not include any edges among $Y_k$s, then both $H^{(0)}$ and $H^{(-j)}$ provide valid estimates of $H$.
\end{rem}
While false positive edges in $H^{(0)}$ are caused by failing to condition on necessary variables, the false positive edges in $H^{(-j)}$ are caused by conditioning on extra variables that are not part of the correct causal order of variables. More generally, the partial ordering information does not lead to a simple estimator that correctly identifies direct causal effects of exposures, $X_k \, (k = 1, \ldots, p)$, on outcomes, $Y_j \, (j = 1, \ldots, q)$. 

Building on the above findings, in the next section we first present a modification of the PC algorithm that leverages the partial ordering information. We then present a general framework for leveraging the partial ordering of variables more effectively, especially when the goal is to learn the graph $H$.

\section{Incorporating Partial Orderings into {\DAG} Learning}
\label{sec:method}

\subsection{A modified PC Algorithm}
Given an ideal test of conditional independence, we can estimate $H$ by first learning the skeleton of $G$, using the PC Algorithm \citep{kalisch_estimating_2007} and then orienting the edges in $H$ according to the partial ordering of nodes. However, such an algorithm uses the ordering information in a \textit{post hoc} way---the information is not utilized to learn the edges in $H$. 

Alternatively, we can incorporate the partial ordering directly into the PC algorithm. 
Recall that PC starts from a fully-connected graph and iteratively removes edges by searching for conditional independence relations that corresponds to graphical d-separations, to this end, it uses separating sets of increasing sizes, starting with sets of size 0 (i.e., no conditioning) and incrementally increasing the size as the algorithm progresses. 
To modify the PC algorithm, we use the fact that if two nodes $j$ and $k$ are non-adjacent in a DAG, then for any of their d-separators, say $S$, the set $S\cap \mathrm{an}(\{j,k\})$ is also a d-separator; see Chapter~6 of \cite{spirtes_causation_2001}.
Given a partial ordering of variables, 
this rule allows excluding from the PC search step any nodes that are in lower layers than nodes $j$ and $k$ under consideration. 
Then, under the assumptions of the PC algorithm, the above modification, which we refer to as PC+, enjoys the same population and sample  guarantees.

Given its construction, the partial ordering becomes more informative in PC+ as the number of layers increases:
in two-layer settings, for learning the edges in $H$, PC+ is exactly the same as PC, and cannot utilize the partial ordering. 
However, as the results in Section~\ref{sec:sim} show, PC+ remains ineffective in graphs with larger number of layers, motivating the development of the framework proposed in the next section.

\subsection{A new framework}

In this section, we propose a new framework for learning {\DAGS} from partial orderings. The proposed approach is motivated by two key observations in Lemmas~\ref{lemma:1} and \ref{lemma:2}: First,  the graphs $H^{(0)}$ and  $H^{(-j)}$ include all true causal relationships from nodes $X_k$ $(k=1, \ldots, p)$ in the first layer to nodes $Y_j$ $(j=1, \ldots, q)$ in the second layer. Second, both graphs may also include additional edges; $H^{0}$ due to not conditioning on parents of $Y_j$ in $\Y$, and $H^{-j}$ due to conditioning on common children of $X_k$ and $Y_j$.

Let $S^{(0)}_j:=\{k:(k\to j)\in H^{(0)}\}$ and $S^{(-j)}_j:=\{k:(k\to j)\in H^{(-j)}\}$. The next lemma, which is a direct consequence of Lemmas~\ref{lemma:1} and \ref{lemma:2}, characterizes the intersection of these two sets.
\begin{lem}\label{lemma:h0hjintersection}
    Let $G$ be a graph admitting the partial ordering $\X\prec \Y$. Then for any $j\in \Y$, and $k\in S^{(0)}_j\cap S^{(-j)}_j$, either $k\in\PA_j$, or $k$ satisfies the followings:
    \begin{itemize}
        \item there exists a path $k\to j'\to \cdots \to j$ with $j'\in \Y$; 
        \item $\mathrm{ch}_j\cap \mathrm{ch}_k\neq \varnothing$.
    \end{itemize}
\end{lem}

Lemma~\ref{lemma:h0hjintersection} implies that even though 
$H^{(0)}_j\cap H^{(-j)}_j$ may contain more edges than the true edges in $H$, these additional edges must be in some special 
configuration. 
Based on this insight, our approach examines the edges in $S^{(0)}_j\cap S^{(-j)}_j$, for each $j\in \Y$, 
to remove the spurious edges efficiently. 
Let $k \in S^{(0)}_j\cap S^{(-j)}_j$ and 
$(k \to j)\notin H$. Suppose, for simplicity, that conditional independencies are faithful to the graph $G$. Then, there must be some set of variables $\mathcal{Z} \subset\X \cup \Y$ such that $X_l \independent Y_j|\mathcal{Z}$. 
In general, to find such a set of variables, we will need to search among subsets of $\X \cup \Y$.
However, utilizing the partial ordering, we can significantly reduce the complexity of this search. 
Specifically, we can restrict the search to 
the \emph{conditional Markov blanket}, introduced next for general {\DAG}s.
\begin{defn}[Conditional Markov Blanket]
Let $v \in V$ and $U \subset V \setminus v$ be arbitrary nodes and subsets in a {\DAG} $G$. The \emph{conditional Markov blanket} of $v$ given $U$, denoted $\cmb_U(v)$, is the smallest set of nodes such that for any other set of nodes $W \subseteq V$,
\begin{equation}\label{eq:cmd_definition}
\PP{v \mid \cmb_U(v), U, W}= \PP{v \mid \cmb_U(v), U}.
\end{equation}
\end{defn}
Our new framework builds on the idea that by limiting the search to the conditional Markov blankets, \emph{given the nodes in the previous layer}, we can significantly reduce the search space and the size of conditioning sets. 
Moreover, the conditional Markov blanket can be easily inferred together with $H^{(0)}$ and $H^{(-j)}$. 
Coupled with the observations in Lemmas~\ref{lemma:1} and \ref{lemma:2}, especially the fact that conditioning on the nodes in the previous layer does not remove true causal edges, these reductions lead to improvements in both computational and statistical efficiency. 

The new framework is summarized in Algorithm~\ref{alg:framework}. 
The algorithm has two main steps: a  \emph{screening loop}, where supersets of relevant edges are identified by leveraging the partial ordering information, and a \emph{searching loop}, similar to the one in the PC algorithm, but tailored to searching over the conditional Markov blankets. 
To this end, we will next show that in order to learn the edges between any node in $\X$ and a node in $\Y$, it suffices to search over subsets of the conditional Markov blanket of nodes in $\Y$.

\begin{algorithm}[t]
	\caption{Learning between-layer edges from partial orderings (PODAG)}\label{alg:framework}
	\SetKwInOut{Input}{Input}
	\SetKwInOut{Output}{Output}
	\Input{Observations from variables $\X = \{X_1,\ldots,X_p\}$ and $\Y = \{Y_1,\ldots,Y_q\}$\\
}
	\Output{A set of edges $\widehat E_{\X \to \Y}$}
	\tcc{Screening loop}
	\For{$j\in \Y$}{
Infer
	${S}^{(0)}_j$, 
	${S}^{(-j)}_j$, and
	$\cmb_{\X}(j)$\;
	}
	$\widehat E\gets \left\{(k, j):j\in\Y,k\in S^{(0)}_j\cap S^{(-j)}_j\right\}$\;
		\tcc{Searching loop}
	\For{$\ell=0,1,\ldots$}{
		\For{$(j,k)\in \widehat E$}{
				\For{$T\subseteq \cmb_{\X}(j)$, $|T|=\ell$}{
					\lIf{$X_k\independent Y_j| \X_{	{S}^{(0)}_j\cap {S}^{(-j)}_j\setminus  \{k\}}\cup \Y_{T}$}{remove $(k,j)$  and break}
			}
		}
		\lIf{no edge can be removed}{break}
	}
	\Return $\widehat{E}$.
\end{algorithm}

\subsection{Graph Identification using the Conditional Markov Blanket}

The next lemma characterizes key properties of the conditional Markov blanket, which facilitate efficient learning of the graph $H$. 
Specifically, we show that if a distribution satisfies the intersection property of conditional independence---i.e., if 
$X\independent Y \mid Z \cup S$ and $ X\independent Z \mid Y \cup S$ then  $X\independent Y \cup Z \mid S$---then all Markov blankets and conditional Markov blankets are unique.

\begin{lem}\label{lemma:CMBproperties}
	Let $V$ be a set of random variables with joint distribution $\mathcal{P}$. Suppose
	the intersection property of conditional independence holds in $\mathcal{P}$.
	Then, for any variable $v\in V$, there exists
	a unique minimal Markov blanket $\mathrm{mb}(v)$. 
	Moreover, for any $U\subset V\setminus v$,
	$\cmb_U(v)$ is also uniquely defined as $\cmb_U(v)=\mathrm{mb}(v)\setminus U$. 
\end{lem}

We next define the faithfulness assumption needed for correct causal discovery from partial orderings. 
This assumption is trivially weaker than the general notion of strong faithfulness \citep{zhang_strong_2002}. 
\begin{defn}[Layering-adjacency-faithfulness]\label{def:LAF}
	Let $\X \prec \Y$ be a layering of random variables $V$ with joint distribution $\mathcal{P}$. 
	We say $\mathcal{P}$ is \emph{layering-adjacency-faithful} to a DAG $G$ with respect to the layering $\X \prec \Y$ if for all $k\in \X$ and $j\in \Y$, if $k \to j \in G$, then $X_k$ and $Y_j$ are (i) dependent conditional on
	$\X_{-k}$; and (ii) dependent conditional on $\X_{-k}\cup T$ for any $T\subseteq \Y_{-j}$.
\end{defn}

Our main result, given below, describes how conditional Markov blankets can be used to effectively incorporate the knowledge of partial ordering into {\DAG} learning and reduce the computational cost of learning causal effects of $\X_k$s on $\Y_j$s. 

\begin{thm}\label{theorem:main}
	For a probability distribution that is Markov and layering-adjacency-faithful with respect to $G$,  a pair of nodes $k\in \X$ and $j\in \Y$ are  non-neighbor in $G$ if and only if there exist a set $T\subseteq \cmb_{\X}(j)$ such that 
	$X_k\independent Y_j|\X_{\big(S_j^{(0)}\cap S^{(-j)}_j\big)\setminus k}\cup \Y_T$.
\end{thm}

It follows directly from Theorem~\ref{theorem:main} that the proposed framework in  
Algorithm~\ref{alg:framework}---i.e., taking the intersection of
$H^{(0)}$ and $H^{(-j)}$ and then searching within conditional Markov blankets to remove additional edges---recovers the correct bipartite DAG $H$. 
\begin{cor}
    Algorithm~\ref{alg:framework} correctly identifies direct causal effects of $X_k$s on $Y_j$s. 
\end{cor}

Next, we show that instead of inferring $H^{(0)}$ and $H^{(-j)}$ separately and taking their interception, 
we can use $H^{(0)}$ to infer $H^{(0)}\cap H^{(-j)}$ directly.
More specifically, noting that Algorithm~\ref{alg:framework} only relies on  $S^{(0)}_j\cap S^{(-j)}_j$ and $\cmb_{\X}(j)$ for each $j\in\Y$, the next lemma shows that given $S^{(0)}_j$, we can learn $S^{(0)}_j\cap S^{(-j)}_j$ and $\cmb_{\X}(j)$ without having to learn $S^{(-j)}_j$ separately. 
This implies that Algorithm~\ref{alg:framework} need not learn the unconditional Markov blankets, but only the conditional Markov blankets. 

\begin{lem}\label{lemma:screening}
    The followings hold for each $j\in \Y$: 
    \begin{itemize}
    \item $S^{(0)}_j = \Big\{k\in\X : X_k\not\independent Y_j \mid \X_{-k}\Big\}$,
        \item  $S^{(0)}_j\cap S^{(-j)}_j=\left\{k\in S^{(0)}_j:X_k\not\independent Y_j \mid \X_{S^{(0)}_j\setminus k}\cup \Y_{-j}\right\}$,
        \item $\cmb_{\X}(j)=\left\{
        \ell\in\Y_{-j}: Y_j\not\independent Y_\ell|\X_{S^{(0)}_j}\cap \Y_{-\{j,\ell\}}
        \right\}.$
    \end{itemize}
\end{lem}

Lemma~\ref{lemma:screening} characterises the conditional dependence sets used in Algorithm~\ref{alg:framework}.
Using these characterizations, we can cast the set inference problems in Algorithm~\ref{alg:framework} as variable selection problems in general regression settings. 
Let
\begin{equation}\label{eq:s1}
    S^{(1)}_j=\left\{Z\in \X_{S^{(0)}_j}\cup \Y_{-j}:Y_j \not\independent Z \mid \X_{S^{(0)}_j}\cup \Y_{-j}\setminus Z \right\}.
\end{equation}
Then, $S^{(0)}_j$ and $S^{(1)}_j$ can be obtained by selecting the relevant variables when regressing $Y_j$ onto $\X$ and $\X_{S^{(0)}_j}\cup \Y_{-j}$, respectively. We can then obtain the two sets used in Algorithm~\ref{alg:framework} from $S^{(1)}_j$:
\[ 
S^{(0)}_j\cap S^{(-j)}_j = S^{(1)}_j\cap \X 
\quad \text{and} \quad
\cmb_{\X}(j) = S^{(1)}_j\cap \Y.
\]

There are multiple benefits to 
directly inferring  $H^{(0)}\cap H^{(-j)}$---i.e., by estimating $S^{(0)}_j$ and $S^{(1)}_j$---instead of separately inferring  $H^{(0)}$ and $H^{(-j)}$.
First, the graph  
$H^{(-j)}$ could be hard to estimate in high-dimensional settings, as learning $H^{(-j)}$ requires  performing tests conditioned on $p+q-2$ variables. 
In contrast, $S_j^{(1)}$ in \eqref{eq:s1} 
only requires performing tests conditioning on at most $\max\left\{p-1,\max_j|S^{(0)}_j|+q-2\right\}$ variables. Moreover, the target 
$H^{(0)}\cap H^{(-j)}$ is often sparse (see Lemma~\ref{lemma:h0hjintersection}) and can thus be efficiently learned in high-dimensional settings. 
In an extreme example, suppose each node in  $\X$ has exactly one outgoing edge into a node in $\Y$, and the first $q-1$ nodes in $\Y$ have as their common child $Y_q$. In this case, $H^{(-j)}$ is fully connected and hard to learn, 
whereas $H^{(0)}\cap H^{(-j)}$ only has $p$ edges.

The second advantage of directly inferring $H^{(0)}\cap H^{(-j)}$ is that, by using the conditional dependence formulation of sets as in
Lemma~\ref{lemma:screening}, we can show that even if the sets $S^{(0)}$ and 
$S^{(1)}_j$---consequently, $S^{(0)}_j\cap S^{(-j)}_j$ and 
$\cmb_{\X}(j)$---are not inferred exactly, the algorithm is still correct as long as no false negative errors are made.

\begin{lem}\label{lem:superset}
Algorithm~\ref{alg:framework} recovers exactly the true DAG $H$ if $S^{(0)}_j$ and $S^{(1)}_j$ are replaced with their arbitrary supersets. 
\end{lem}

We refer to the framework proposed in Algorithm~\ref{alg:framework} for \emph{learning {\DAGS} from partial ordering} the PODAG framework. 
In the next section, we discuss specific algorithms for learning direct casual effects of $X_k$s on $Y_j$s. These algorithms utilize the fact that given the partial ordering of nodes in $G$, the conditional Markov blanket of each node $j \in \Y$ can be found efficiently by testing for conditional dependence of $Y_j$ and $Y_{j'}, j' \ne j$ after adjusting for the effect of the nodes in the first layer.

\section{Learning High-Dimensional {\DAGS} from Partial Orderings}\label{sec:estimation}
Coupled with a consistent test of conditional independence, the general framework of Section~\ref{sec:method} can be used to learn {\DAGS} from any layering-adjacent-faithful probability distribution (Definition~\ref{def:LAF}). This involves two main tasks for each for $j\in \Y$: (a) obtaining a consistent estimate of the set of relevant variables, $S^{(0)}_j$, and (b) obtaining a consistent estimate of the conditional Markov blanket, $\mathrm{cmb}_{X}(j)$, and the conditioning set in top layer, $S^{(0)}_j\cap S^{(-j)}_j$. By Lemma~\ref{lemma:screening}, this task is equivalent to learning $S^{(1)}_j$ in \eqref{eq:s1}. 
In this section, we illustrate these steps by presenting  efficient algorithms that leverage the partial orderings information to learn bipartite {\DAGS} corresponding to linear structural equation models (SEMs). 
In this case, existing results on variable selection consistency of network estimation methods can be coupled with a proof similar to that of the PC Algorithm \citep{kalisch_estimating_2007} to establish consistency of the sample version of the proposed algorithm for learning high-dimensional {\DAGS} from partial orderings. 

Suppose, without loss of generality, that the observed random vector
$\W = \X \cup \Y =(W_1,\dots,W_{p+q})$ is centered. In a structural equation model, $\W$ then solves an equation system $
	W_j = f_{j}(\W_{\PA_j},\varepsilon_j)
	$ for $j=1,\ldots, p+q$,
where $\varepsilon_j$ are independent random variables with mean zero and
$f_{j}$ are unknown functions. 
In linear SEMs, each $f_{jk}$ is linear:
\begin{equation}\label{eqn:linearSEM}
	W_j = \sum_{k\in \PA_j}\beta_{jk}W_k + \varepsilon_j,\qquad j=1,\ldots, p+q.
	\end{equation}
Specialized to bipartite graphs,  Equation~\ref{eqn:linearSEM} can be written compactly as
\begin{equation}\label{eqn:linearmatrix}
\begin{pmatrix}
\X \\ \Y
\end{pmatrix} =\begin{pmatrix}
A&0\\B&C
\end{pmatrix}\begin{pmatrix}
\X \\ \Y
\end{pmatrix}+ \varepsilon,    
\end{equation}
where, as in \eqref{eq:adjmat}, $A$, $B$ and $C$ are $p\times p$, $q \times p$ and $q \times q$ coefficients matrices, respectively. 
Our main objective is to estimate the matrix $B$. To this end, we propose statistically and computationally efficient procedures for estimating the sets $S^{(0)}_j$ and $S^{(1)}_j$ for all $j\in \Y$, which, as discussed before, are the main ingredients needed in Algorithm~\ref{alg:framework}. We will also discuss how these procedures can be applied in high-dimensional settings, when $p,q \gg n$.

We next show that whenever the screening loop of Algorithm~\ref{alg:framework} is successful---that is, when it returns a supergraph of $H$---the searching loop can consistently recover the true $H$. 
We start by stating our assumptions, which are the same as those used to establish the consistency of the PC algorithm \citep{kalisch_estimating_2007}. 
The only difference is that our `faithfulness' assumption---Assumption~\ref{ass:gaussian_faith}---is weaker than the corresponding assumption for the PC algorithm. The consequences of this relaxation are examined in Section~\ref{sec:sim}. 

\begin{assm}[Maximum reach level]\label{ass:gaussian_reach}
	Suppose there exists some $0<b\leq 1$
	such that 
    $h_n:=\max_{j\in \Y}|\ADJ_j\cap \Y|=O(n^{1-b})$ and  $m_n=\max_{j\in \Y}|S^{0}_j\cap S^{(-j)}_j|=O(n^{1-b})$.
\end{assm}
\begin{assm}[Dimensions]\label{ass:gaussian_dim}
    The dimensions $p$, $q$ satisfies 
    $pq^{m_n+1}=O(\exp(c_0n^{\kappa}))$
    for some $0<c_0<\infty$  and $0\leq \kappa<1$.
\end{assm}
\begin{assm}[$\lambda$-strong layering-adjacency-faithfulness]\label{ass:gaussian_faith}
The distribution of $(\X,\Y)$ is multivariate Gaussian and the partial correlations satisfy
    \[
    \inf_{j\in\Y, k\in\X,T\subseteq \Y \setminus \{j\}}\Big\{\left|\rho\left(Y_j,X_k \mid \Y_T\cup \X_{-k}\right)\right|:  \rho\left(Y_j,X_k \mid \Y_T\cup \X_{-k}\right)\neq 0 \Big\}\geq c_n,
    \]
    \[
     \sup_{j\in\Y, k\in\X,T\subseteq \Y\setminus \{j\}}\Big\{\left|\rho\left(Y_j,X_k \mid \Y_T\cup \X_{-k}\right)\right|\Big\}\leq M<1 \text{ for some } M,
    \]
    where $c_n^{-1} = O(n^{d})$ for some $0< d < \frac{1}{2}\min(b,1-\kappa)$ with $\kappa$ defined in Assumption~\ref{ass:gaussian_dim} and $b$ as in Assumption~\ref{ass:gaussian_reach}.
\end{assm}
The next result establishes the consistency of the searching step. 
\begin{thm}[Searching step using partial correlation]\label{thm:search_linGaussSEM}
    Suppose Assumptions \ref{ass:gaussian_reach}--\ref{ass:gaussian_faith} hold.
	Let the event 
    $$\mathcal{A}\left(\widehat S^{(0)},\widehat S^{(1)}\right)=\left\{\forall j\in \Y : 
    \widehat S^{(0)}_j\supseteq  S^{(0)}_j, 
    \widehat S^{(1)}_j\supseteq  S^{(1)}_j, 
    |\widehat S^{(0)}_j\cap \widehat S^{(1)}_j|\leq n\right\}$$ 
    denote the success of the screening step. 
	Then, there exists some sequence of thresholds
	$\alpha_n\to 0$ as $n\to \infty$ such that the output $\widehat H$ of Algorithm~\ref{alg:framework} with test of partial correlation in  searching loop satisfies
	\[
	\PP{\widehat H= H|\mathcal{A}\left(\widehat S^{(0)},\widehat S^{(1)}\right)}
	= 1- O\left(\exp\left(-Cn^{1-2d}\right)\right).
	\]
\end{thm}

Theorem~\ref{thm:search_linGaussSEM} shows that, with high probability, the searching loop returns the correct graph as long as the screening loop makes no false negative errors. 
In the following sections, we discuss various options for the screening step, namely, the estimation of $S^{(0)}$ and $S^{(1)}$ defined in Lemma~\ref{lemma:screening}:
\begin{align*}
S^{(0)}_j &= \{k\in\X : X_k\not\independent Y_j \mid \X_{-k}\}, \\
S^{(1)}_j &= \{Z\in \X_{S^{(0)}_j}\cup \Y_{-j}:Y_j \not\independent Z \mid \X_{S^{(0)}_j}\cup \Y_{-j}\setminus Z \}.
\end{align*}

\subsection{Screening in low dimensions}\label{LDscreen}
Given multivariate Gaussian observations, 
the most direct way to identify the sets satisfying the conditional independence relations in $S^{(0)}_j$ and $S^{(1)}_j$ is to use 
the sample partial correlation $\hat\rho$
and reject the hypothesis of conditional independence when $|\hat\rho|>\xi$ for some suitable threshold $\xi$.
The next results establishes the screening guarantee for such an approach in low-dimensional settings. 

\begin{prop}[Screening with partial correlations]\label{lem:screen_pcor}
    Suppose Assumptions \ref{ass:gaussian_reach}--\ref{ass:gaussian_faith} hold and $n\gg p+q$. 
Let 
\begin{align*}
    \widehat S^{(0)}_j &= \left\{k\in\X:|\hat\rho(Y_j,X_k\mid\X_{-k})|>c_n/2\right\}, \\
    \widehat S^{(1)}_j &= \left\{Z\in\X_{S^{(0)}_j}\cup \Y_{-j}:|\hat\rho(Y_j,Z \mid \X_{S^{(0)}_j}\cup \Y_{-j}\setminus Z)|>c_n/2\right\}.
\end{align*}
Then, 
    $
    \PP{\mathcal{A}\left(\widehat S^{(0)},\widehat S^{(1)}\right)}=1-O\left(\exp\left(-Cn^{1-2d}\right)\right).
    $
\end{prop}

\subsection{Screening in high dimensions}\label{HDscreen}
The conditioning sets in $S^{(0)}_j$ and $S^{(1)}_j$---namely, $\X_{-k}$ and $\X_{S^{(0)}_j}\cup \Y_{-j}\setminus Z$---involve large number of variables. Therefore, in high dimensions, when $p,q \gg n$, it is not feasible to learn $S^{(0)}_j$ and $S^{(1)}_j$ using partial correlations. 
To overcome this challenge, and using the insight from Lemma~\ref{lemma:screening}, we treat the screening problem as selection of non-zero coefficients in linear regressions. 
This allows us to use tools from high-dimensional estimation, including various regularization approaches, for the screening step. 

The success of the screening step requires the event $\mathcal{A}\left(\widehat S^{(0)}_j,\widehat S^{(1)}_j\right)$---defined in Theorem~\ref{thm:search_linGaussSEM}---to hold with high probability. That is, we need to identify supersets of $S^{(0)}_j$ and $S^{(1)}_j$ such that $\left|\widehat S^{(0)}_j \cap \widehat S^{(1)}_j\right| \le n$. 
A simple approach to screen relevant variables is \textit{sure independence screening} \citep[SIS,][]{fan_sure_2008}, which selects the variables with largest marginal association at a given threshold $t \in (0,1)$.
Defining the SIS estimators 
\begin{align*}
    \widehat S^{(0)}_{j,\mathrm{SIS}}(t) &= 
    \left\{
1\leq k\leq p: |X_k^\top Y_j| \textnormal{ is among the first $\lceil tn \rceil$ largest of all } \X^\top Y_j
\right\}\\
\widehat S^{(1)}_{j,\mathrm{SIS}}(t) &= 
    \left\{
Z\in \X_{\widehat S^{(0)}}\cup \Y_{-j}: |Z^\top Y_j| \textnormal{ is among the first $\lceil tn \rceil$ largest of all } (\X_{\widehat S^{(0)}}\cup \Y_{-j})^\top Y_j
\right\},
\end{align*}
we can obtain the following result following  \citet{fan_sure_2008}.
\begin{prop}\label{prop:SIS}
    Suppose Assumptions~\ref{ass:gaussian_reach}--\ref{ass:gaussian_faith} hold
    with $\kappa < 1-2d$, and the maximum eigenvalue of $(\X,\Y)$ is lower bounded by $c n^\xi$. 
    If $2d+\xi<1$, then there exists some $\delta<1-2d-\xi$ such that when $t\sim cn^{-\delta}$ with $c>0$, we have for some $C>0$, 
    \[
    P\left[\mathcal{A}\left(\widehat S^{(0)}_{\mathrm{SIS}}(t),\widehat S^{(1)}_{\mathrm{SIS}}(t)\right)\right]=1-O\left(\exp\left(-Cn^{1-2d}/\log n\right)\right).
    \]
\end{prop}

By using simple marginal associations, SIS offers a simple and efficient strategy for screening; however, it may result in large sets for the searching step. 
An alternative to SIS screening is to select the sets $\widehat S_j^{(0)}$ and $\widehat S_j^{(1)}$ using two penalized regressions. For instance, using the lasso penalty, $\widehat S_j^{(0)}$ for $j\in\Y$ can be found as
\begin{align*}
    \widehat\gamma_j &:= \argmin_{\gamma\in \RR^{p-1}}\frac{1}{2n}\left\|Y_j - \gamma^\top \X\right\|_2^2+\lambda^{(0)}_n\left\|\gamma_j\right\|_1,  &\widehat S^{(0)}_j = \left\{k:\widehat\gamma_{jk}\neq 0\right\}.
\end{align*}
Similarly, $S_j^{(1)}$ can be learned using a lasso regression on a different set of variables; specifically, given any set $S\subseteq \X$, we define
\begin{align*}
    \widehat\theta_j(S) &:= \argmin_{\theta \in \RR^{|S|+q-1}} \frac{1}{2n}\left\|Y_j - \theta^\top \left[\X_{S}, \Y_{-j}\right] \right\|_2^2+\lambda^{(1)}_n\left\|\theta_j\right\|_1,  &\widehat S^{(1)}_j = \left\{k:\widehat\theta_{jk}\left(\widehat S^{(0)}_j\right)\neq 0\right\}.
\end{align*}
The next result establishes the success of the screening step when using lasso estimators in high dimensions. 
\begin{prop}[Screening with lasso]\label{thm:screen_lasso}
Suppose Assumptions~\ref{ass:gaussian_reach}--\ref{ass:gaussian_faith} hold. 
Also assume that the minimal eigenvalue of  $\mathrm{Cov}(\X\cup\Y)$ is larger than some constant $\Gamma_{\min}$, and 
there exists some $\xi>0$ such that  $\mathrm{Var}(Z|\X\cup\Y\setminus Z)>\xi$ for all $Z\in \X\cup\Y$. 
Assume
$s_n=\max_j|\mathrm{mb}(j)|=O\left(n^{1-a}\right)$ 
and the rate parameters in Assumptions~\ref{ass:gaussian_reach} and \ref{ass:gaussian_dim} satisfy 
$\kappa\leq\min(a,b)$.
Then, lasso estimators with penalization level lower bounded with  
$\lambda_n^{(0)}\asymp\sqrt{2\log p/n}$ and $\lambda_n^{(1)}\asymp\sqrt{2\log (p+q)/n}$ satisfy
$\PP{\mathcal{A}\left(\widehat S^{(0)},\widehat S^{(1)}\right)}\to 1$.
\end{prop}

The above two high-dimensional screening approaches---SIS and lasso---are two extremes: SIS obtains the estimates of $S^{(0)}_j$ and $S^{(1)}_j$ without adjusting for any other variables, whereas lasso adjusts for all other variables (either from the first layer or the two layers combined). 
These two extremes represent examples of broader classes of methods based on marginal or joint associations. They differ with respect to computational and sample complexities of screening and searching steps: 
Compared with SIS, screening with lasso may result in smaller sets for the searching step but may be less efficient for screening. Nevertheless, both approaches are valid choices for screening under milder conditions than those needed for consistent variable selection.
This is because the additional assumption required for the screening step is implied by the layering-adjacency-faithfulness assumption, which, as discussed before, is weaker than the strong faithfulness assumption needed for the PC algorithm (see also Section~\ref{sec:sim}). 
While the above two choices seem natural, other intermediates, including conditioning on small sets \citep[similar to ][]{sondhi_reduced_2019} can also be used for screening.

\subsection{Other Approaches}\label{otherscreening}
Together with Proposition~\ref{lem:screen_pcor}, \ref{prop:SIS}, or \ref{thm:screen_lasso}, Theorem~\ref{thm:search_linGaussSEM} establishes the consistency of Algorithm~\ref{alg:framework} for linear SEMs, in low- and high-dimensional settings. 
While we only considered the simple case of Gaussian noise, Algorithm~\ref{alg:framework} and the results in this section can be extended to linear SEMs with sub-Gaussian errors \citep{harris_pc_2013}. Moreover, many other regularization methods can be used for screening instead of the lasso.
The only requirement is that the method is screening-consistent. For instance, inference-based procedures, such as debiased lasso
\citep{geer_asymptotically_2014, zhang2014confidence, javanmard_hypothesis_2014} can also fulfill the requirement of our screening step. 

Similar ideas can also be used for learning {\DAGS} from other distributions and/or under more general SEM structures. In Appendix~\ref{sec:CAM}, we outline the generalization of the proposed algorithms to causal additive models \citep{buhlmann2014cam}.
The Gaussian copula transformation of \citet{liu_nonparanormal_2009,liu_high-dimensional_2012} and \citet{xue_regularized_2012} can be similarly used to extend the proposed ideas to non-Gaussian distribution. We leave the detailed exploration of these generalizations to future research.

\section{Extensions and Other Considerations}\label{sec:extensions}

\subsection{Learning Edges within Layers}\label{sec:inlayeredges}

In the previous sections, we focused on the problem of learning edges between two layers, $\X$ and $\Y$. 
In particular, Algorithm~\ref{alg:framework} provides a framework that first finds a superset of the true edges, and then uses a searching loop to remove the spurious edges. 
The same idea can be applied to learning edges within the second layer, $\Y$: 
As suggested by Lemma~\ref{lemma:CMBproperties},
for each $j\in\Y$, all of its adjacent nodes are contained in $S^{(1)}_j$. 
This suggests that we can learn edges in $\Y$ by simply modifying Algorithm~\ref{alg:framework} to run the searching loop on $\left\{(k,j):k\in  S^{(1)}_j,j\in\Y\right\}$ instead of only on those between $\X$ and
$\Y$.
After recovering the skeleton, edges among $\Y$ can be oriented using Meek's orientation rules \citep{meek_causal_1995}.
Given correct d-separators and skeleton,  the rules can orient as many edges as possible without making mistakes.

In order to successfully recover within layer edges from observational data, we need to assume faithfulness among these layers.
Following \cite{ramsey_adjacency-faithfulness_2006}, we characterize the faithfulness condition required to learn the {\DAG} via constraint-based algorithms as two parts:
(a) \emph{adjacency faithfulness}, which means that neighboring nodes are not associated with conditional independence; and (b) \emph{orientation faithfulness}, which means that parents in any v-structure are not conditionally independent given any set containing their common child. 
The orientation faithfulness can be defined in our context as follows. 
\begin{assm}[Within-layer-faithfulness]\label{ass:within-layer-faithful}
For all adjacent pairs $(j,j')\in\Y$, it holds that $Y_{j'}\not\independent Y_j \mid \X\cup T$ for any set $T\subseteq \Y\setminus\{j,j'\}$. Also, for any unshielded triple $(i,j,k)$ with $j\in\Y$, if $i \to j\leftarrow k$,  then the variables corresponding to $i$ and $k$ are dependent given any subset of $\X\cup\Y\setminus\{i,k\}$ that contains $j$; otherwise, the variables corresponding to $i$ and $k$ are dependent given any subset of $\X\cup\Y\setminus\{i,k\}$ that does not contains $j$.
\end{assm}

Given the additional information from partial ordering, to describe the graphical object learned by the new framework, we need to introduce a new notion of equivalence. This is because the Markov equivalent class of {\DAGS} can be reduced by using the known partial ordering. 
For example, if the only conditional independent relation among three variables $X_i,X_j,X_k$ is $X_i\independent X_k|X_j$, but we know $X_i\prec \{X_j,X_k\}$, then the edge $X_j\to X_k$ is identifiable. We define partial-ordering-Markov-equivalence simply as Markov equivalence restricted to partial ordering. 
This equivalence class can be represented by a maximally oriented partial {\DAG}, or maximal PDAG  \citep{perkovic_complete_2018}.
With this notion, we can describe the target of learning of within- and between-layer edges as learning the maximal PDAG of the true $G$ given the background information $\X\prec\Y$.

\begin{lem}[within-layer edges]\label{lemma:withinlayer}
    Suppose the conditions for Theorem~\ref{thm:search_linGaussSEM} and Assumption~\ref{ass:within-layer-faithful} hold. Then,  Algorithm~\ref{alg:framework} with an additional orientation step by Meek's rules recovers the maximal PDAG of $G$ given the background information $\X\prec\Y$.
\end{lem}

\subsection{Directed Graphs with Multiple Layers}\label{sec:moresets}

The theory and algorithm developed in Section~\ref{sec:estimation}  can be extended 
to scenarios with multiple layers.
To facilitate this extension, we introduce a general representation of the problem. 
Suppose $V$ is a random vector following some distribution Markov to a graph $G=(V,E)$. 
Suppose $G$ admits a partial-ordering $\mathcal{O}=\{V_1\prec\cdots\prec V_L\}$ where $V = \cup_{\ell=1}^L V_\ell$.
Parallel to the notation in 2-layer case, we define, for each $j\in V_\ell$, $1\leq\ell\leq L$,
\begin{align*}
    S^{(0)}_j&:= \begin{cases} \varnothing& \text{ if } j\in V_1\\
    \left\{k\in \cup_{i=1}^{\ell-1} V_{i}: V_k\not\independent V_j|\left(\cup_{i=1}^{\ell-1}V_i\right)\setminus \{k\} \right\} &\text{ otherwise},
    \end{cases}\\
    S^{(1)}_j&:= \left\{k\in  V_{\ell}\setminus\{j\}: V_k\not\independent V_j|V_{S^{(0)}_j}\cup V_\ell \setminus \{k,j\} \right\}.
\end{align*}
The next lemma generalizes Lemma~\ref{lemma:h0hjintersection}.

\begin{lem}\label{lemma:h0hjintersection_multi}
    For any graph $G$ admitting the partial ordering $\mathcal{O}=\{V_1\prec\cdots\prec V_L\}$, the followings hold:
    \begin{itemize}
        \item For each $j\in V_\ell$, $2\leq \ell\leq r$,
        a node $k\in S^{(0)}_j\cap S^{(1)}_j $ if and only if $(k,j)\in G$ or $(k,j)\notin G$ but
        there exists a path $k\to j'\to \cdots \to j$ with $j'\in V_\ell$ and $\CH_j\cap \CH_k\cap V_\ell\neq \varnothing$.
        \item For 
        each $j\in V_\ell$, $1\leq \ell\leq r$,
        a node $k\in S^{(1)}_j\cap V_\ell$ if and only if $k\in\ADJ_j$ or $k\notin \ADJ_j$ but
         $\CH_j\cap \CH_k\cap V_\ell\neq \varnothing$.
    \end{itemize}
\end{lem}

\begin{algorithm}[t]
	\caption{{\DAG} Learning from Partial Orderings (Multiple Layer Setting)}
	\label{alg:framework_multi}
	\SetKwInOut{Input}{Input}
	\SetKwInOut{Output}{Output}
	\Input{Observations from random variables $W \sim \mathcal{P}_G$,\\ Partial Ordering $\mathcal{O}=\{V_1\prec\cdots\prec V_L\}$ where $V = \cup_{i=1}^L V_\ell$
}
	\Output{A estimated edge set of $G$}
	\lFor{$j\in V$}{infer
	${S}^{(0)}_j$ and 
	${S}^{(1)}_j$
	}
	$\widehat E_{\textnormal{between}}\gets \left\{(k, j):2\leq \ell\leq L, j\in V_\ell,k\in S^{(0)}_j\cap S^{(1)}_j\right\}$\;
	$\widehat E_{\textnormal{within}}\gets \left\{(k, j):1\leq \ell\leq L, j\in V_\ell,k\in S^{(1)}_j\cap V_\ell\right\}$\;
	\For{$d=0,1,\ldots$}{
	    \For{$\ell=0,1,\ldots,L$}{
		\For{$(k,j)\in \widehat E_{\textnormal{between}}\cup \widehat E_{\textnormal{within}}$, $j\in V_\ell$}{
				\For{$T\subseteq S^{(1)}_j\cap V_\ell$, $|T|=d$}{
					\lIf{$V_k\independent V_j|
					V_{\left(S^{(0)}_j\cap S^{(1)}_j\right)\cup T\setminus \{k\}}$}{remove $(k,j)$  and break}
				}
			}
		}
		\lIf{no edge can be removed}{break}
	}
	Orient all edges in $\widehat E_{\textnormal{between}}$ by $\mathcal{O}$; then apply Meek's rules to orient edges in $\widehat E_{\textnormal{within}}$\;
	\Return $\widehat E_{\textnormal{between}}\cup \widehat E_{\textnormal{within}}$.
\end{algorithm}

Lemma~\ref{lemma:h0hjintersection_multi} suggests a general framework for utilizing any layering information to facilitate {\DAG} learning. The multi-layer version of the algorithm is presented in  Algorithm~\ref{alg:framework_multi}.
The faithfulness condition required for the success of this framework is given below for the set of variables $\W$.
\begin{assm}[Layering-faithfulness]\label{ass:layering-faith}
For a graph $G=(V,E)$ admitting a partial-ordering $\mathcal{O}=\{V_1\prec\cdots\prec V_L\}$, we say a distribution $\mathcal{P}$ is layering faithful to $G$ with respect to $\mathcal{O}$ if the followings hold:
\begin{itemize}
    \item \emph{Adjacency faithfulness:} For all non-adjacent pairs $j,k$ with $j\in V_\ell$, $k\in V_{\ell'}$, $\ell \geq \ell'$, it holds that $W_j\not\independent W_k|\W_{\cup_{i=1}^{\ell-1}V_i\setminus \{k\}}$ and
    $W_j\not\independent W_k|\W_{(\cup_{i=1}^{\ell-1}V_i)\cup T \setminus \{k\}}$ for all $T\subseteq V_\ell\setminus\{j,k\}$.
    \item \emph{Orientation faithfulness:} 
    For all unshielded triples $(i,j,k)$ such that  $j,k$
    are in the same layer $V_s$ and $i$ is in some previous layer, if the configuration of the path $(i,j,k)$ is $i\to j\leftarrow k$ then $W_i\not\independent W_k \mid \W_T\cup\{j\}$ for all $T\subseteq \cup_{t=1}^{s}V_t\setminus \{i,k\}$; otherwise $W_i\not\independent W_k \mid \W_T$ for all $T\subseteq \cup_{t=1}^{s}V_t\setminus \{i,j,k\}$.
\end{itemize}
\end{assm}
We note that orientation faithfulness is only needed for triplets when at least two of the three nodes are in the same layer. 
This is because otherwise the orientation is already 
implied by partial ordering. 

\begin{thm}\label{thm:multilayer}
Under Assumption~\ref{ass:layering-faith}, the population version of
Algorithm~\ref{alg:framework_multi} recovers the maximal PDAG of $G$.
\end{thm}

\subsection{Weaker Notions of Partial Ordering}
Algorithm~\ref{alg:framework_multi} can be successful even if ordering information is not available for some variables. This is because the idea of using a screening loop and a searching loop to reduce computational burden can be more generally applied. 
Suppose that for every variable $j\in V$ there exist two sets 
$T^{\prec}_j \prec j$ and $T^{\succ}_j \succ j$. Then, we  slightly modify the construction of $S_j^{(0)}$ and $S_j^{(1)}$: for each $j\in V_\ell$,
\begin{align*}
    S^{(0)}_j&:= 
    \left\{k\in T^{\prec}_j: W_k\not\independent W_j|T^{\prec}_j\setminus \{k\} \right\}, \\
    S^{(1)}_j&:= \left\{k\in V\setminus(T^{\prec}_j\cup T^{\succ}_j\cup\{j\}): W_k\not\independent W_j|\W_{S^{(0)}_j}\cup\left( \W\setminus(T^{\prec}_j\cup T^{\succ}_j\cup\{j,k\})\right) \right\}.
\end{align*}
It is easy to see that Lemma~\ref{lemma:h0hjintersection_multi} hold with this generalized notion. 
Specifically, this means that 
Algorithm~\ref{alg:framework_multi} can also handle the setting in which $V$ can be partitioned into disjoint sets $V_1,\ldots, V_L, V'$, such that the layering information among $V_1,\ldots, V_L$ is known ($V_1\prec V_2\prec\cdots\prec V_L$) and $V'$ contains nodes with no partial ordering information. 
This generalization allows the algorithm to be applied in settings, where e.g., $V_1,\ldots, V_{L-1}$ represent confounders, exposures and mediators, $V_L$ represents the  outcomes and $V'$ may correspond to variables for which no partial ordering information is available.

\section{Numerical Studies}\label{sec:perfanal}

In this section, we compare our proposed framework in  Algorithm~\ref{alg:framework}, PODAG, with the PC algorithm and its ordering-aware variant introduced in  Section~\ref{sec:method}, PC+. Our numerical analyses comprise three aspects: a comparison of faithfulness conditions; simulation studies comparing {\DAG} estimation accuracy of various methods; and an analysis of quantitative trait loci (eQTL). 

\subsection{Comparison of Faithfulness Conditions}
To assess the faithfulness conditions of different algorithms, we examine the minimal absolute value of detectable partial correlations. 
Specifically, for a constraint-based structure learning algorithm that tests conditional independencies in the collection of tuples  
$\mathcal{L}$, the quantity 
\[\rho^*_{\min}(\mathcal{L})=
\min_{\{j,k,S\}\in \mathcal{L}}\left\{|\rho(j,k \mid S)|:\rho(j,k \mid S)\neq 0\right\}
\]
can be viewed as the strength of faithfulness condition required to learn the \DAG. 
A small value of $\rho_{\min}^*(\mathcal{L})$ indicates a weaker separation between signal and noise for the method $\mathcal{L}$, making it  harder to recover the correct {\DAG} using sample partial correlations. 
On the other hand, a larger value of $\rho_{\min}^*(\mathcal{L})$ 
indicates superior statistical efficiency.

We randomly generated $R=100$ {\DAGS} with $20$ nodes and two expected edges per node. 
For each \DAG, we constructed a linear Gaussian SEM with parameters drawn uniformly from $(-1,-0.1)\cup (0.1,1)$.
We inspected three algorithms, PC, PC+, and PODAG, and computed their corresponding $\rho_{\min}^*(\mathcal{L}_{\text{PC}}), \rho_{\min}^*(\mathcal{L}_{\text{PC+}}), \rho_{\min}^*(\mathcal{L}_{\text{PODAG}})$. 
We also counted the number of conditional independence tests performed by each method as a measure of computational efficiency. 
The results are shown in Figure~\ref{fig:faithful}.
It can be seen that the faithfulness requirement for PC is stronger than PC+, which is in turn stronger than PODAG. The results also show that PC and PC+ require more tests of conditional independence, further reducing their statistical efficiency. 
Noticeably, though PC+ utilizes the partial ordering information, its computational and statistical efficiency are subpar compared with PODAG. 

\begin{figure}
    \centering
    \includegraphics[width=0.9\linewidth]{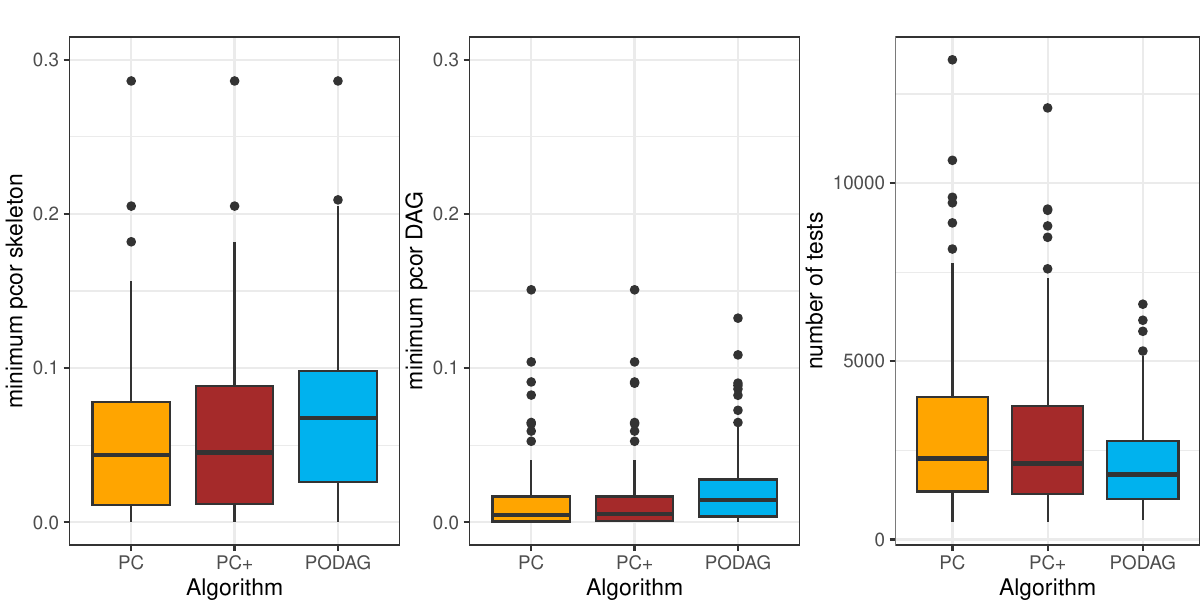}
    \caption{
    Comparison of $\rho_{\min}^*(\mathcal{L}_{\text{PC}}), \rho_{\min}^*(\mathcal{L}_{\text{PC+}})$ and  $\rho_{\min}^*(\mathcal{L}_{\text{PODAG}})$
    for recovering the  skeleton (left) and the entire {\DAG} (middle).
    The number of conditional independence tests is shown on the right panel.}
    \label{fig:faithful}
\end{figure}

\subsection{{\DAG} Estimation Performance}\label{sec:sim}
We compare the performance of the proposed PODAG algorithm in learning {\DAGS} from partial ordering with those of the PC and PC+ algorithms. To this end, we consider graphs with $50,100,150$ vertices and $3$ expected edges per node. Edges are randomly generated within and between layers, with higher probability of connections between layers. 
A weight matrix for linear SEM is generated with respect to the graph, with non-zero parameters drawn uniformly from $(-1,-0.1)\cup (0.1,1)$. 
We generate $n=500$ samples using SEMs with Gaussian errors. 
The vertex set is split into either $L=2$ or $L=5$ sets, corresponding to equal size layers.

\begin{figure}
\includegraphics[width=0.99\linewidth]{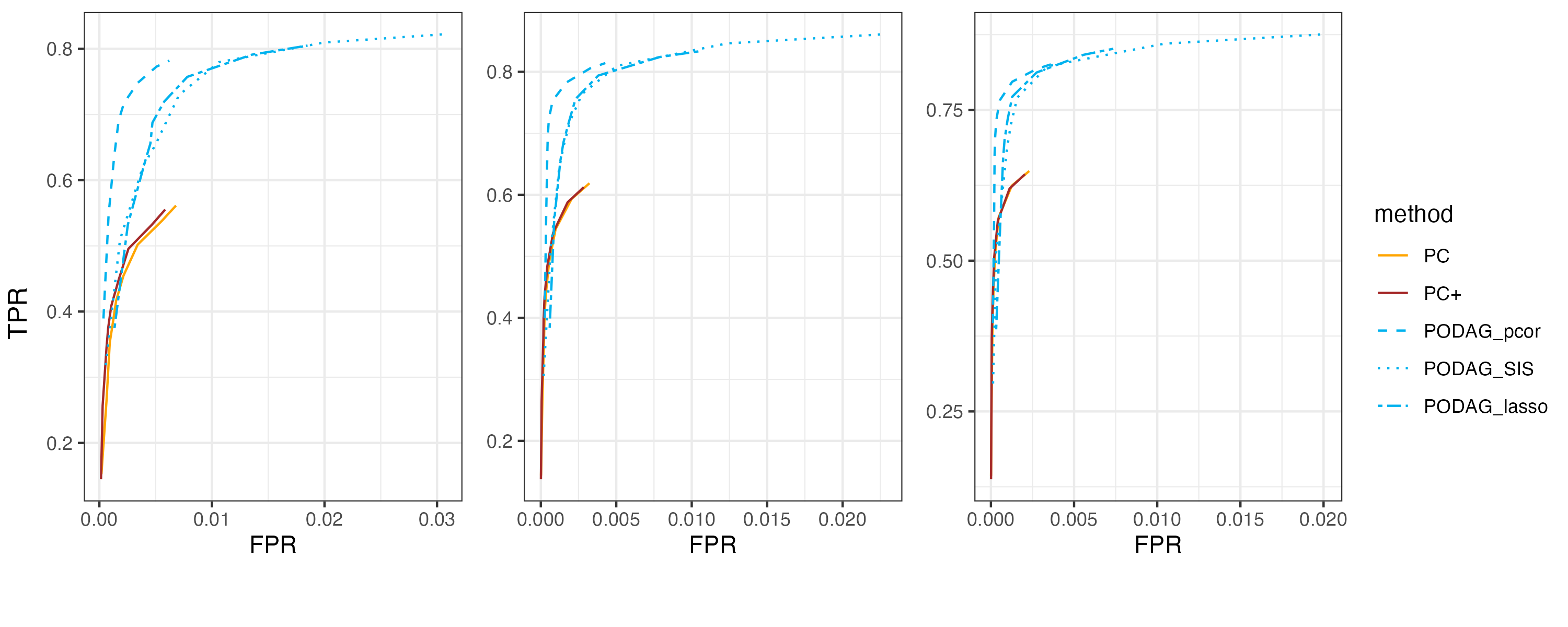}
\includegraphics[width=0.99\linewidth]{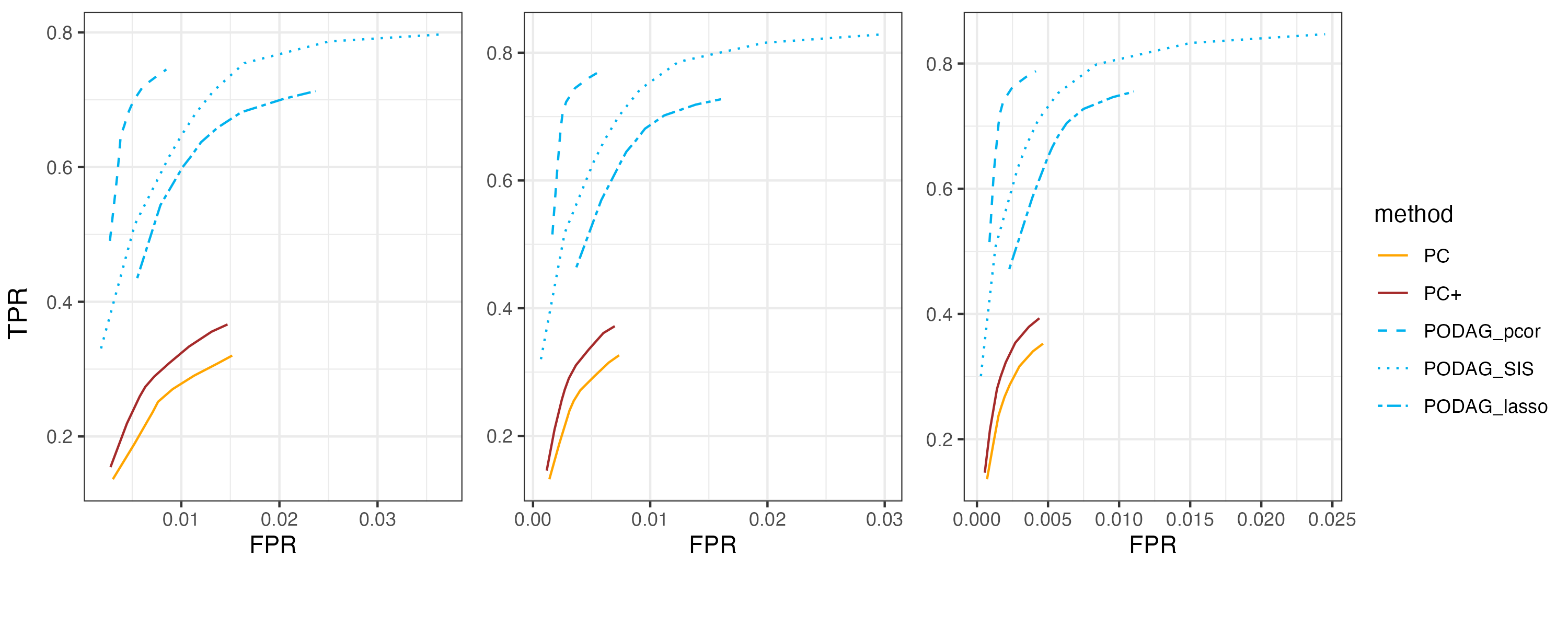}
    \caption{Average false positive and true positive rates (FPR, TPR) for learning {\DAGS} from partial orderings. The edges are generated from random ER graphs with $50$ (\emph{left}), $100$ (\emph{middle}), and $150$ (\emph{right}) nodes and 3 expected edges per node. Samples are drawn from Gaussian SEM with sample size $n=500$. Partial ordering information supplied to the algorithms in the form of two layers ({\it first row}) or five layers ({\it second row}).} \label{fig:precisionrecall_n500_allp}
\end{figure}

Figure~\ref{fig:precisionrecall_n500_allp} shows the average true positive and false positive rates (TPR, FPR) of the proposed PODAG algorithm compared with PC and PC$+$. Three variants of PODAG are presented in the figure, based on different screening methods in the first step: using sure independence screening (SIS), lasso and partial correlation. 
Regardless of the choice of screening, the results show that by utilizing the partial ordering, PODAG offers significant improvement over both the standard PC and PC+. 
In these simulations, partial correlation screening seems to offer the best performance. Moreover, the gap between high-dimensional screening methods (SIS and lasso) with partial correlation tightens as the dimension grows, especially for two-layer networks. 

The plots in Figure~\ref{fig:precisionrecall_n500_allp} correspond to different number of variables---50, 100, 150---equally divided into either $L=2$ or $L=5$ layers. While the performances are qualitatively similar across different number of variables, the plots show significant improvements in the performance of PODAG as the number of layers increases. This is expected, as larger number of layers ($L=5$) corresponds to more external information, which is advantageous for PODAG.  
The results also suggest that SIS screening may perform better than lasso screening in the five-layer network. However, this is likely due to the choice of tuning parameter for these methods. In the simulations reported here, for SIS screening, marginal associations were screened by calculating p-values using Fisher's transformation of correlations; an edge was then included if the corresponding p-value was less than 0.5 (to avoid false negatives). 
On the other hand, the tuning parameter for lasso screening was chosen by minimizing the Akaike information criterion (AIC), which balances prediction and complexity. Other approaches to selection of tuning parameters may change the size of the sets obtained from screening, and hence the relative performances of the two methods.

\subsection{Quantitative Trait Loci Mapping}\label{sec:data}
We illustrate the applications of PODAG by applying it to learn the bipartite network of interactions among DNA loci and gene expression levels. This problem is at the core of expression Quantitative trait loci (eQTL) mapping, which is a powerful approach for identifying sequence variants that alter gene functions \citep{nica_expression_2013}. While many tailored algorithms have been developed to identify DNA loci that affect gene expression levels, 
the primary approach in biological application is to focus on cis-eQTLs (i.e., local eQTLs), meaning those DNA loci around the gene of interest (e.g., within 500kb of the gene body). This focus is primarily driven by the limited power in detecting trans-eQTLs (i.e., distant eQTLs) by a genome-wide search.
State-of-the-art algorithms also incorporate knowledge of DNA binding cites and/or emerging omics data \citep{kumasaka2016fine, keele2020integrative} to obtain more accurate eQTL maps \citep{sun2012statistical}.

Regardless of the focus (i.e., cis versus trans regulatory effects), the computational approach commonly used to infer interactions between DNA loci and gene expression levels is based on the marginal association between expression level of each gene and the variability of each DNA locus in the sample. This approach provides an estimate of the $H^{0}$ graph defined in \eqref{eq:H0}. Since  the activities of other DNA loci are not taken into account when inferring the effect of each DNA locus on each gene expression level, the resulting estimate is a superset of $H^{0}$, unless an uncorrelated subset of DNA loci, for instance after LD-pruning \citep{fagny2017exploring}, is considered. 

\begin{figure}[t]
\centering
\includegraphics[width=.99\textwidth]{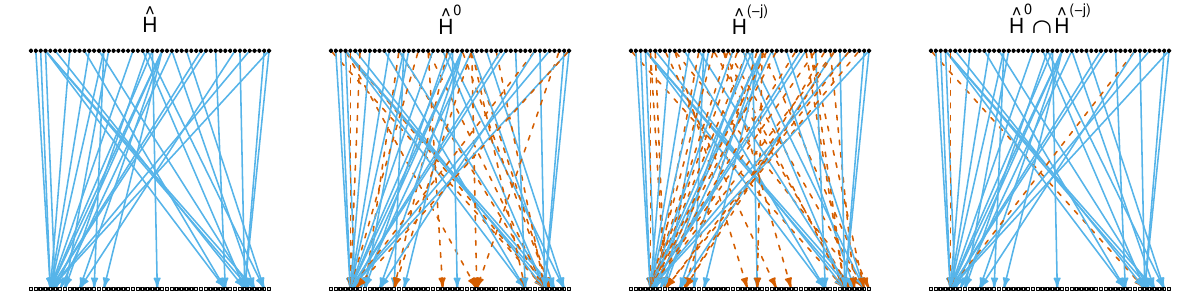}
\caption{Estimated quantitative trait mappings for yeast for $p=50$ randomly selected DNA loci and $q=50$ randomly selected gene expression levels. The plots show the estimated bipartite graph using the proposed PODAG algorithm, $\widehat H$; the $\widehat H^0$ estimate commonly obtained in eQTL analyses, by regressing the expression levels (nodes in the second layer) on the DNA loci; the $\widehat H^{(-j)}$ estimate defined in \eqref{eq:Hminusj} and the intersection of edges in $\widehat H^0$ and $\widehat H^{(-j)}$, which is the starting point of PODAG.}\label{fig:eQTL}
\end{figure}

To illustrate the potential advantages of the proposed PODAG algorithm in this setting, we use a yeast expression data set \citep{brem_landscape_2005}, containing eQTL data for 112 segments, each with 585 shared markers and 5428 target genes. As an illustrative example, and to facilitate visualization, we randomly select $p=50$ DNA loci and $q=50$ gene expression levels. 
Our analysis focuses on the direct effect of DNA markers on gene expression levels, corresponding to the bipartite network $H$, which can be estimated using PODAG by leveraging the  partial ordering of the nodes in the two layers. 

The results are presented in Figure~\ref{fig:eQTL}. As expected, the commonly used estimate of $H^{0}$, and the estimate of $H^{(-j)}$ in \eqref{eq:Hminusj}---obtained by also regressing on other gene expression levels---both contain additional edges compared with the PODAG estimate, $\widehat H$. The intersection of the first two networks, $\widehat H^0 \cap \widehat H^{(-j)}$ (discussed in Section~\ref{sec:challenge}), reduces these false positives. It also provides an effective starting point for PODAG, demonstrating the efficiency of the proposed algorithm.

\section{Discussion}\label{sec:disc}
We presented a new framework for leveraging partial ordering of variables in learning the structure of directed acyclic graphs (\DAGS). 
The new framework reduces the number of conditional independence tests compared with the methods that do not leverage shis information. It also requires a weaker version of the faithfulness assumption for recovering the underlying DAG.  

The proposed framework consists of two main step: a \emph{screening} step, where a superset of interactions is identified by leveraging the partial ordering; and a \emph{searching} step, similar to that used in the PC algorithm \citep{kalisch_estimating_2007}, that also leverages the partial ordering. 
The framework is general and can accommodate various estimation procedures in each of these steps. 
In low-dimensional linear structural equation models, partial correlation is the natural choice. Non-parametric approaches for testing for conditional independence \citep{azadkia2021simple, huang2022kernel, shi2021azadkia, shah2020hardness} can generalize this strategy to infer {\DAGS} from partial orderings under minimal distributional assumptions. In this paper, we considered sure independence screening (SIS) and lasso as two options for screening in high dimensions.  
However, many alternative algorithms can be used for estimation in high dimensions. As an illustration, the algorithm is extended in Appendix~\ref{sec:CAM} to causal additive models \citep{buhlmann2014cam}. 

An important issue for the screening step in high dimensions is the choice of tuning parameter for the estimation procedure. Commonly used approaches for selecting tuning parameters either focus on prediction or model selection. Neither is ideal for our screening, where the goal is to obtain a small superset of the relevant variables, which, theoretically, requires no false negatives. To achieve a balance, we used the Akaike information criterion (AIC) to tune the lasso parameter in our numerical studies. Recent advances in high-dimensional inference \citep{geer_asymptotically_2014, zhang2014confidence, javanmard_hypothesis_2014} could offer a viable alternative to choosing the tuning parameter and can also lead to more reliable identification of network edges.

\section*{Acknowledgements}
This work was partially supported by grants from the National Science Foundation and the National Institutes of Health. We thank Dr. Wei Sun for helpful comments on an earlier draft of the manuscript. 

\bibliographystyle{abbrvnat}
\bibliography{bib}

\clearpage
\appendix

\section{Proofs}\label{sec:proofs}

\subsection{Proof of Results in Section~\ref{sec:problem}}

\begin{proof}[Proof of Lemma~\ref{lemma:1}]
The claim in (i) follows directly from the definition of $d$-separation. In particular, if $X_k \ra Y_j \in G$ then the path from $X_k$ to $Y_j$ will not be $d$-separated by $\X_{- k}$, and hence $Y_j \dep X_k \mid \X_{- k}$.

To prove (ii), note that $\{ X_1, \ldots, X_q \} \prec \{ Y_1, \ldots, Y_p \}$ implies that $Y_{j_1} \ra \cdots \ra Y_{j_0}$ cannot include any $X_k$s. Thus, without loss of generality, suppose that 
\[
	Y_{j_1} \ra \cdots \ra Y_{j_0} \equiv Y_{j_1} \ra Y_{j_2} \ra \cdots \ra Y_{j_m} \ra Y_{j_0},
\]
where $m = 1$ is permitted.
Now, considering that $X_{k_0}$ is an ancestor of $Y_{j_0}$, by faithfulness of $\mathcal{P}$ with respect to $G$, in order for any set $S \subset V \backslash \{k_0, j_0\}$ to $d$-separate $X_{k_0}$ and $Y_{j_0}$, it must include at least one of the variables $Y_{j_1}, Y_{j_2}, \ldots, Y_{j_m}$. Noting that the construction used in $H$ only adjusts for $\X_{- k_0}$, and does not adjust for any $Y_j$s, the path from $X_{k_0}$ to $Y_{j_0}$ will not be $d$-separated by $\X_{- k_0}$. It follows from the definition of $d$-separation that $X_{k_0}$ and $Y_{j_0}$ are conditionally dependent given $\X_{- k_0}$, which means that $X_{k_0} \ra Y_{j_0} \in H^{(0)}$.
\end{proof}

\begin{proof}[Proof of Lemma~\ref{lemma:2}]
The result follows again from the definition of $d$-separation for {\DAGS}, and the faithfulness of $\mathcal{P}$. 
In particular, if $X_{k} \ra Y_{j} \in G$, $\{ X_{- k}, Y_{- j} \}$ does not $d$-separate $X_{k}$ from $Y_{j}$, and hence, $Y_j \dep X_k \mid \{ X_{- k}, Y_{- j} \}$.

Next, if $X_{k_0}$, $Y_{j_0}$ and $Y_{j_1}$ form an open collider in $G$, i.e., if $X_{k_0} \ra Y_{j_1} \la Y_{j_0}$ and $X_{k_0} \nra Y_{j_0}$, then the above argument implies that $X_{k_0} \ra Y_{j_1} \in H^{(-j)}$. On the other hand, since $Y_{j_1}$ is a common descendent of  $X_{k_0}$ and $Y_{j_0}$, $d$-separation implies that $X_{k_0} \dep Y_{j_0} \mid Y_{j_1}$, or more generally, $X_{k_0} \dep Y_{j_0} \mid \{ X_{- k_0}, Y_{- j_0} \}$. Thus, $X_{k_0} \ra Y_{j_0} \in H^{(-j)}$.
\end{proof}

\subsection{Proof of Results in Section~\ref{sec:method}}

\begin{proof}[Proof of Lemma~\ref{lemma:CMBproperties}]
    Suppose $\mathrm{mb}'(v)$ and $\mathrm{mb}''(v)$ are two distinct valid minimal Markov blankets (MB) of $v$. 
    Let $J= \mathrm{mb}'(v)\cap \mathrm{mb}''(v)$, $K=\mathrm{mb}'(v)\setminus J$, $L=\mathrm{mb}''(v)\setminus J$ and
    $W = V\setminus (J\cup K\cup L\cup v)$.
    By the intersection property, 
    $v\independent K|J\cup L$ 
    and 
    $v\independent L|J\cup K$ implies $v\independent K\cup L|J$.
    By the elementary formula of conditional independence,
    $v\independent K\cup L|J$ and $v\independent W\cup L|J\cup K$
    implies 
    $v\independent W\cup K\cup L|J$, and therefore
    $J$ is also a valid MB. But this is only possible when $C=\mathrm{mb}'(v)=\mathrm{mb}''(v)$.
    Hence the minimal MB is unique. 
    
    Next we show $\cmb_U(v)=\mathrm{mb}(v)\setminus U$ in two steps. 
    First, we show $\mathrm{mb}(v)\subseteq \left(\cmb_U(v)\cup U\right)$.
    We write $W = \mathrm{mb}(v)\cap \left(\cmb_U(v)\cup U\right)$, $Z=\mathrm{mb}(v)\setminus W$, $S=(\cmb_U(v)\cup U)\setminus W$. Then, by the definitions of conditional Markov blanket and Markov blanket,  $v\independent Z|W,S$ and 
    $v\independent S|W,Z$.
    By the intersection property,
    $v\independent Z,S|W$ which establishes $W$ as a valid Markov Blanket. However, since 
    $W\subseteq \mathrm{mb}(v)$ and MB is minimal,  this is only possible when 
    $W= \mathrm{mb}(v)$ and $Z=\varnothing$, i.e., $\mathrm{mb}(v)\subseteq\cmb_U(v)\cup U$.
    Finally we show $\cmb_U(v)\subseteq \mathrm{mb}(v)$.
    We write $T=\cmb_U(v)\cap \mathrm{mb}(v)$. Since $T\cup U\supseteq \mathrm{mb}(v)$, it  must holds that
    $v\independent S|T,U$, which implies that 
    $T$ is also a valid conditional Markov blanket. However, 
    since
    $T\subseteq \cmb_U(v)$
    and  the conditional Markov blanket is minimal, this is only possible when $T= \cmb_U(v)$, i.e.,
    $\cmb_U(v)\subseteq \mathrm{mb}(v)$. This completes the proof. 
\end{proof}

\begin{proof}[Proof of Theorem~\ref{theorem:main}]
Under layering-adjacency-faithfulness, if $k\in\X$ and
$j\in\Y$ are adjacent, then
$X_k\not\independent Y_j|\X_{-k}$
and $X_k\not\independent Y_j|\X_{-k}\cup \Y_{-j}$, 
so $k\in S^{(0)}_j\cap S^{(-j)}_j$. Therefore $\widehat E$ contains all true edges. 

By the intersection property, it holds that 
$X_k\independent Y_j|\X_{\left\{S^{(0)}_j\cap S^{(-j)}_j\right\}\setminus k}\cup \Y_T$, and hence, $ X_k\independent Y_j|\X_{-k}\cup \Y_T$.
Under layering-adjacency-faithfulness,
this implies no false negative errors can be made by the searching loop. 
Suppose $k\in \X$
and $j\in\Y$ are non-adjacent, 
then we must have $X_k\independent Y_j \mid \PA_j$,
which implies 
$X_k\independent Y_j \mid \X_{\left\{S^{(0)}_j\cap S^{(-j)}_j\right\}\setminus k}\cup \Y_{T}$ with the set $T=\PA_j \cap\Y$.
This construction provides one $d$-separator, and hence there could not be any false positive errors either.
\end{proof}

\begin{proof}[Proof of Lemma~\ref{lem:superset}]
Since both sets are replaced by their supersets, $\widehat E$ is still a superset of true edges. 
Under layering-adjacency-faithfulness,
for any adjacent pairs $k\in\X$ and $j\in\Y$ and any set $T\subset\Y_{-j}$, it holds that $X_k\not\independent Y_j|\X_{-k}\cup T$, so the searching loop in Algorithm~\ref{alg:framework} will not make any false negative errors.
Therefore, we only need to show that all additional edges can be removed by the searching loop. 
For any nonadjacent $k\in\X$ and $j\in\Y$, let $W\subseteq \X$ be an arbitrary set satisfying  $W\supseteq \left(S^{(0)}\cap S^{(-j)}_j\right)$. 
Then, 
$X_k\independent Y_j|\X_{W\setminus k}\cup \Y_{\pa_j\cap\Y}$. The existence of this separator implies that the edge between $k$ and $j$ can be removed by the searching loop. 
\end{proof}

\subsection{Proof of Results in Section~\ref{sec:estimation}}
We first state an auxiliary lemma. 
\begin{lem}[Partial correlation and regression coefficient]\label{lem:zerocoeff=zeropcor}
	Let $\beta^S_j$ be the linear regression coefficients regressing 
	a set of variables $\X_S$ onto $X_j$, and let $\rho(i,j|S)$ denote the partial correlation of $X_i$ and $X_j$ given $\X_S$. Then, 
	\[
	\beta_{ji}^{S\cup i}= 0 \Leftrightarrow \rho(i,j|S)=0\Leftrightarrow
	\Sigma_{i,j}-\Sigma_{S,i}^\top \Sigma_{S,S}^{-1}\Sigma_{S,j}=0.
	\]
\end{lem}
\begin{proof}
	Denote the block covariance matrix 
	$\Sigma_{S\cup \{i,j\}, S\cup \{i,j\}}$ as 
	\[
	\begin{pmatrix}
	A & B & C\\ B^\top & D & E\\ C^\top & E^\top & F
	\end{pmatrix} := \begin{pmatrix}
	\Sigma_{S,S} & \Sigma_{S,i}  & \Sigma_{S,j} \\ 
	\Sigma_{i,S}  & \Sigma_{i,i}  & \Sigma_{i,j} \\ \Sigma_{j,S}  & \Sigma_{j,i}  & 
	\Sigma_{j,j} 
	\end{pmatrix}
	\]
	For the first term, we have $\beta_{j}^{S\cup i} = \left(\Sigma_{S\cup i, S\cup i}\right)^{-1} \Sigma_{S\cup i,j}$, and hence 
	\[
	\beta_{ji}^{S\cup i} = \left(D - B^\top A^{-1} B\right)^{-1} \left(E-  B^\top  A^{-1}C\right).
	\]
	For the second term, we have 
	$\rho(i,j|S)=-\frac{\Omega_{ij}}{\sqrt{\Omega_{ii}\Omega_{jj}}}$, where $\Omega$ is the precision matrix: $\Omega=\Sigma^{-1}$. Therefore, $\rho(i,j|S)=0$ if and only if 
	$ [(\Sigma_{S\cup \{i,j\}, S\cup \{i,j\}})^{-1}]_{i,j}=0$. Concretely, 
\[
[(\Sigma_{S\cup \{i,j\}, S\cup \{i,j\}})^{-1}]_{i,j}=-\beta_{ji}^{S\cup i} \left(F - 
C^\top A^{-1}  C - (-C^\top A^{-1} B + E^\top )  
\beta_{ji}^{S\cup i} \right)^{-1}.
\]
	Both quantities equal zero if and only if  $\Sigma_{i,j}-\Sigma_{S,i}^\top \Sigma_{S,S}^{-1}\Sigma_{S,j}=0$.
\end{proof}

\begin{proof}[Proof of Theorem~\ref{thm:search_linGaussSEM}]
    Let $\alpha_n = 2(1-\Phi(n^{1/2}c_n/2))$ and denote $h_n=\max_{j\in\Y}|\widehat S^{(0)}_j\cap \widehat S^{(1)}_j|$.
    We first suppose  Algorithm~\ref{alg:framework} is continued until some level
     $m_n'\geq m_n$ such that 
     $m_n'+h_n=O(n^{1-b})$ as in Assumption~\ref{ass:gaussian_reach}.
    By Lemma~3 of \citet{kalisch_estimating_2007}, the probability of any type I error in edge $k,j$ with conditioning set $T\subset \Y\setminus{j}$, denoted as $E_{k,j|T}^{\mathrm{I}}$, is bounded by an exponential term:
    \[
    \sup_{k\in\X,j\in\Y,T\subset  \Y\setminus {j}, |T|\leq m_n'}\PP{E^{\mathrm{II}}_{k,j|T}|\mathcal{A}}\leq O(n-m_n'-h_n)\exp\left(-C_3(n-m_n'-h_n)c_n^2\right),
    \]
    and the probability of false negative error is also bounded by an exponential term:
    \[
    \sup_{k\in\X,j\in\Y,T\subset  \Y\setminus {j}, |T|\leq m_n'}\PP{E^{\mathrm{II}}_{k,j|T}|\mathcal{A}}\leq O(n-m_n'-h_n)\exp\left(-C_4(n-m_n'-h_n)c_n^2\right),
    \]
    for finite constants $C_3$ and $C_4$. 
    In our algorithm, $|\{k,j,T\}:|T|\leq m_n|=O\left(pq^{1+m_n'}\right)=O\left(\exp\left(c_0n^\kappa\right)\right)$.
    So the total error probability can be bounded: 
 \begin{align}\label{eqn:PCerr} \nonumber
    &\PP{\text{an error occurs in Algorithm~\ref{alg:framework} with p-cor testing}|\mathcal{A}}\\ \nonumber
    &\leq O\left(\exp(c_0n^\kappa)\right)
    O\Big[(n-m_n'-h_n)\exp\left(-C_5(n-m_n'-h_n)c_n^2\right)\Big]\\ \nonumber
    &\leq O \Big[\exp\left(
   c_0n^\kappa  + \log(n-m_n'-h_n)- C_5\left(n^{1-2d}-m_n'n^{-2d}-h_nn^{-2d}\right)
    \right)\Big]\\ 
    &=O\left(\exp\left(-Cn^{1-2d}\right)\right).
 \end{align}
By \eqref{eqn:PCerr}, with high probability, the sample version of searching loop makes no mistakes up to search level $m_n'$. 
Then by the reasoning in Lemma~5 of \citet{kalisch_estimating_2007}, the searching loop also has high probability of terminating at the same level as the population version, which is either $m_n-1$ or $m_n$. This completes the proof. 
\end{proof}

\begin{proof}[Proof of Lemma~\ref{lem:screen_pcor}]
    We apply Lemma~3 of \citet{kalisch_estimating_2007}
    with the same notion of $E^{\mathrm{II}}_{j,\widehat S^{(0)}_j}$ and
    $E^{\mathrm{II}}_{j,\widehat S^{(1)}_j}$ as in Theorem~\ref{thm:search_linGaussSEM}. Then, by the same argument, it holds that 
    \[
    \sup_{j\in\Y}\PP{E^{\mathrm{II}}_{j,\widehat S^{(0)}_j}}\leq O(n+1-p)\exp\left(-C_4(n+1-p)c_n^2\right),
    \]
    \[
    \sup_{j\in\Y}\PP{E^{\mathrm{II}}_{j,\widehat S^{(1)}_j}}\leq O(n+2-p-q)\exp\left(-C_4(n+2-p-q)c_n^2\right).
    \]
    Since there are in total $pq+(p+q-1)q$ many tests, the probability of errors can be bounded by
    \[
    \PP{\neg \mathcal{A}}\leq q(2p+q-1)(n+2-p-q)\exp\left(-C'(n+2-p-q)c_n^2\right)=O\left(\exp(-C'n^{1-2d})\right),
    \]
    using the fact that $n\gg p+q$. 
\end{proof}

\begin{proof}[Proof of Proposition~\ref{prop:SIS}]
By Theorem 1 of \cite{fan_sure_2008},
for each selection problem, SIS has an 
error probability of $O\left(\exp\left(-Cn^{1-2d}/\log n\right)\right)$.
Applying a union bound over 2$p$ problems, the error probability is 
\begin{align*}
O\left(\exp\left(n^\kappa-Cn^{1-2d}/\log n\right)\right)
&= O\left(\exp\left(n^{1-2d}/(n^{1-2d-\kappa}-C/\log n)\right)\right) \\
&= O\left(\exp\left(-Cn^{1-2d}/\log n\right)\right).
\end{align*}
\end{proof}

\begin{proof}[Proof of Proposition~\ref{thm:screen_lasso}]
To complete this proof, we first show that under the Gaussian construction, 
all design matrices satisfy the Restricted Eigenvalue (RE) condition. 
We then show that this condition implies bounded error of
the lasso estimate. 
Lastly, we show that under  Assumption~\ref{ass:gaussian_faith} the minimal non-zero coefficients are large enough to guarantee successful screening. 

The RE condition is a technical condition for lasso to have bounded error. 
In particular,  a  design matrix $\mathcal{D}\in\RR^{n\times r}$ satisfies RE 
over a set $S\subseteq [r]$ with
parameter $\eta$ if $\frac{1}{n}\norm{\mathcal{D}\Delta}_2^2\geq \eta\norm{\Delta}_2^2$ for all $\Delta\in\RR^r$ such that 
$\norm{\Delta_{[r]\setminus S}}_1\leq 3 \norm{\Delta_S}_1$. 
Let $\Gamma_{\min}(\cdot)$ be the minimum eigenvalue function and $\rho^2(\cdot)$
 the maximum diagonal entry. 
It is shown in Theorem 7.16 of \cite{wainwright_high-dimensional_2019} that 
for any random design matrix $\mathcal{D}\sim N(0,\Sigma_D)$,
there exists universal positive constants $c_1<1<c_2$
such that 
$$
\frac{\norm{\mathcal{D}w}_2^2}{n}\geq c_1\norm{\sqrt{\Sigma_D}w}_2^2-c_2\rho^2(\Sigma_D)\frac{\log r}{n}\norm{w}_1^2
$$
for all $w\in\RR^r$, with probability at least $1-\frac{\exp(-n/32)}{1-\exp(-n/32)}$.
This inequality implies
RE condition with $\eta=\frac{c_1}{2}\Gamma_{\min}(\Sigma_D)$
uniformly for any subset $S$ of cardinality at most 
$|S|\leq \frac{c_1}{32c_2} \frac{\Gamma_{\min}(\Sigma_D)}{\rho^2(\Sigma_D)}\frac{n}{\log r}$. 
Our required rates of $h_n=O(n^{1-b})$ and $s_n=O(n^{1-a})$ satisfy this sparsity requirement.
Therefore, under our assumption on minimum eigenvalues, 
the matrix $\X\cup\Y$ satisfied RE with some constant $\eta$
with probability at least $1-\frac{\exp(-n/32)}{1-\exp(-n/32)}$.
Consequently, the design matrices used to learn $S^{(0)}_j$ and $S^{(1)}_j$, 
i.e., $\X$ and 
$\{(\X\cup\Y_{-j})\}_{j\in\Y}$, 
all satisfy the RE condition. 

Next we apply Theorem~7.13 of \citet{wainwright_high-dimensional_2019}.
In particular, 
for any solution of the lasso problem with regularization level lower bounded as $\lambda_n\geq 2 \norm{\mathcal{D}^\top \omega /n}_{\infty}$ where $\mathcal{D}$ is the design matrix, it holds that 
$\norm{\hat\beta - \beta^*}_2\leq \frac{2}{\eta}\sqrt{s}\lambda_n$. 
In our case, this guarantees all the $\ell_2$ errors of estimation is bounded, in particular
\[
\max_{j\in\Y}\left(\max\{\norm{\hat\gamma_j-\gamma_j^*}_2,\norm{\hat\theta_j-\theta_j^*}_2\}\right)\leq \frac{2}{\eta}\sqrt{s_n}\lambda_n.
\]
For any Gaussian designs $\mathcal{D}\in\RR^{n\times r}$ with columns standardized to $\max_j\frac{\norm{\mathcal{D}_j}_2}{\sqrt{n}}\leq C$,
the Gaussian tail bound guarantees $\PP{\norm{\frac{\mathcal{D}^\top \omega }{n}}_{\infty}>C\sigma(\sqrt{\frac{2\log r}{n}}+\delta)}\leq 2\exp\left\{-\frac{n\delta^2}{2}\right\}$.

To learn $S^{(0)}_j$ with design matrix $\X$,  we set $\lambda_n\asymp\sqrt{2\log p/n}$, 
then the errors (in $\ell_2$ and $\ell_\infty$) are in the order of $O\left(\sqrt{\frac{2s_n\log p}{n}}\right)=O\left(n^{1-b+\kappa}\right)$ with probability at least $1-2\exp\left\{-\frac{n\delta^2}{2}\right\}$.
A union bound over all regression problems bounds all errors with probability $1-4p\exp\left\{-\frac{n^2\delta}{2}\right\}$.
Similarly, To learn $S^{(1)}_j$ with design matrix $\{(\X\cup\Y_{-j})\}_{j\in\Y}$, our choice of $\lambda_n\asymp\sqrt{2\log (p+q)/n}$ guarantees union bound over all regression problems with probability $1-4(p+q)\exp\left\{-\frac{n^2\delta}{2}\right\}=1-O\left(\exp\left\{c_0 n^\kappa - \frac{\delta}{2}n^2\right\}\right) = 1-O\left(\exp\left\{-\frac{\delta}{2}n^2\right\}\right)$.

Lastly, we show that Assumption~\ref{ass:gaussian_faith} implies a beta-min condition for the regression problems. Note that
\begin{align*}
    \gamma_{kj}&=\rho(j,k\mid\X_{-k})\sqrt{\mathrm{Var}(Y_j\mid\X)\mathrm{Var}(X_k|\X_{-k}\cup Y_j)},\\
    \theta_{kj}&=\rho(j,k\mid\X_{-k}\cup \Y_{-j})\sqrt{\mathrm{Var}(Y_j|\X_{-\{j,k\}})\mathrm{Var}(X_k|\X_{-\{j,k\}})},
\end{align*}
and the last terms are both bounded below by some constant, so that  $\min_{\gamma_{jk}\neq0}|\gamma_{kj}|>c'_n$ and $\min_{\theta_{jk}\neq0}|\theta_{kj}|>c'_n$ for some ${c'_n}^{-1}=O(n^{-d})$ from Assumption~\ref{ass:gaussian_dim}. 

The beta-min condition ensures that lasso successfully selects a superset of the  relevant covariate with the error bound specified above. 
Specifically, we have shown that with probability at least $1-O(e^{-\frac{n\delta^2}{2}})-\frac{e^{-n/32}}{1-e^{-n/32}}$, 
lasso solution recovers a superset of  $S^{(0)}_j$ and $S^{(1)}_j$ for all $j\in\Y$. 
\end{proof}

\subsection{Proof of Results in Section~\ref{sec:extensions}}

\begin{proof}[Proof of Lemma~\ref{lemma:withinlayer}]
    The proof is identical to that of Theorem ~\ref{theorem:main}. The only additional piece needed for successful recovery of  PDAG is the orientation step. The known partial ordering is a form of background knowledge of edge orientation, which, combined with Meek's rules of orientation, returns the maximal PDAG \cite[see][Problem~D]{meek_causal_1995}.
\end{proof}

\begin{proof}[Proof of Lemma~\ref{lemma:h0hjintersection_multi}]
    The first statement follows directly from Lemma~\ref{lemma:h0hjintersection} (treating $\cup_{i=1}^{\ell-1}V_i$ as $\X$ and $V_\ell$ as $\Y$). 
    The second statement follows from Lemma~\ref{lemma:CMBproperties} that $S^{(1)}_j\cap V_\ell = \text{mb}(j)\cap V_\ell$.
\end{proof}

\begin{proof}[Proof of Theorem~\ref{thm:multilayer}]
    By Lemma~\ref{lemma:h0hjintersection_multi}, the result of screening step is a superset of the edges in $G$. 
    Under the adjacency faithfulness assumption, all conditional independencies checked by the algorithm corresponds to d-separation and therefore 
     the correct skeleton is recovered.
    Finally, given correct d-separation relations, the orientation rules are complete and maximal \citep{perkovic_complete_2018}.
\end{proof}

\section{Causal Additive Models}\label{sec:CAM}

In this section, we outline how the proposed framework can be applied to the more flexible class of causal additive models (CAM) \citep{buhlmann_cam_2014}, that is, SEMs that are jointly additive:
\begin{equation}\label{eqn:CAM}
	W_j = \sum_{k \in \PA_j}f_{jk}(W_k) + \epsilon_j\qquad j=1,\ldots, p+q.
\end{equation}
Like the linear SEM in the main paper, we will discuss the searching and screening steps separately. 

For the searching step, we can use a general test of conditional independence. 
Here, we adopt the framework of \citet{chakraborty_nonparametric_2022}, 
based on conditional distance covariance (CdCov). 
Let two vectors $T\in\RR^a$ and $U\in\RR^b$, their \text{CdCov} given $Z$ is defined as
\[
\text{CdCov}(T,U|Z) =\frac{1}{c_ac_b}\int_{\RR^{a+b}}\frac{|f_{T,U|Z}(t,s)-f_{T|Z}(t)f_{U|Z}(s)|^2}{\norm{t}_a^{1+a}\norm{s}_{b}^{1+b}}.
\]
We define $\rho^*(T,U|Z)=\EE{\text{CdCov}^2(T,U|Z)}$. 
It is easy to see that $\rho^*(T,U|Z)=0$ if and only if $T\independent U|Z$. 

Let $K_H(\omega)=|H|^{-1}K(H^{-1}\omega)$ be some kernel function where $H$ is the diagonal matrix determined by bandwidth $h$ and denote $K_{iu}=K_H(Z_i-Z_u)$.
We also write $d^T_{ij}=\norm{T_i-T_j}_a$ and 
$d^U_{ij}=\norm{U_i-U_j}_b$. Define 
$d_{ijkl}:=(d^T_{ij}+d^T_{kl}-d^T_{ik}-d^T_{jl})(d^U_{ij}+d^U_{kl}-d^U_{ik}-d^U_{jl})$
and the symmetric form
$d^S_{ijkl}=d_{ijkl}+d_{ijlk}+d_{ilkj}$, 
We can use a plug-in estimate for $\rho^*(T,U|Z)$:
\[
\widehat \rho^*(T,U|Z):=\frac{1}{n}\sum_{u=1}^{n}\Delta_{i,j,k,l;u}
\, \text{ where } \,
\Delta_{i,j,k,l;u}:=\sum_{i,j,k,l}\frac{K_{iu}K_{ju}K_{ku}K_{lu}}{12(\sum_{m=1}^{n}K_{mu})^4}d_{ijkl}^S.
\]
Following the derivation of Theorem~3.3 in \cite{chakraborty_nonparametric_2022}, 
we can obtain the following result, which we state here without proof. 
\begin{prop}[Searching with CdCov test]\label{thm:cdcov}
    Suppose Assumption~\ref{ass:gaussian_reach} and \ref{ass:gaussian_dim} holds with $\kappa<1/4$.
    Assume  
    there exists $s_0>0$ such that for all 
    $0\leq s<s_0$, 
    \[
        \max_{W\in\X\cup\Y} \mathbb{E}\left[\exp(sW^2)\right] <\infty,
    \]
    and the kernel function $K(\cdot)$
    used to compute $\widehat \rho$
    is non-negative and uniformly bounded over its support.
    Assume in addition the faithfulness condition that there exists some $c_n$ such that
    \[
    \inf_{j\in\Y,k\in\X,T\subseteq \Y_{-j}}\left\{|\rho^*(j,k\mid\Y_T\cup\X_{-k})|: \rho^*(j,k\mid\Y_T\cup\X_{-k})\neq 0 \right\}>c_n,
    \]
    where $c_n^{-1}=O(n^d)$ with 
    $d<\frac{1}{4}-\frac{1}{2}\kappa$. 
    Then, conditioned on $\mathcal{A}(\widehat S^{(0)},S^{(1)})$,
    the searching loop in
    Algorithm~\ref{alg:framework}
    using test of CdCOV  returns the correct edge set with probability at least $1-O\left(\exp(-n^{1-2\gamma-2d})\right)-O\left(\exp(-n^\gamma)\right)$.
\end{prop}

Next we discuss the screening step 
for nonlinear SEM models. 
We consider the family of functions $f_{uv}^{(r)}(x_u)=\Psi_{uv}\beta_{uv}$, where $\Psi_{uv}$ is a $n\times r$ matrix whose columns are basis functions used to model the additive components $f_{uv}$, and 
$\beta_{uv}$ is a $r$-vector containing the associated effects.
We write $\phi_{uvt}$ as the $t$-th 
coefficient in $\Phi_{uv}$. 
We denote $\Psi_{S}$ as the concatenated basis functions in 
$\{\Psi_{uv}:u\in S\}$.
Denote $\Sigma_{S,S}=(n^{-1}\Psi_S^\top \Psi_S)$.

In this model $k$ has no direct causal effect on $j$ if and only if
$f_{jk}\equiv 0$. Under some regularity conditions, there exists 
a truncation parameter $r$ large enough such that $f_{uv}\equiv 0\Leftrightarrow f_{uv}^{(r)}\equiv 0\Leftrightarrow \beta_{uv}=[0,\ldots,0]^\top$
for all $u$ and $v$. 
In high-dimensional problems, we
 estimate the coefficients with 
$\ell_1/\ell_2$ norm penalization. 
Concretely, for a node $u\in V$, we can estimate $\widehat S^{(0)}_j$ and $\widehat S^{(1)}_j$ by solving the following problems
\begin{align*}
	\widehat S^{(0)}_j &= \left\{
	k: \hat f^{(r)}_{jk} \equiv 0: \hat f^{(r)}_{jk} =\argmin_{\{f_{jl}\}_{l\in \X}\in \mathcal{F}^{(r)}} 
\norm{Y_j - \sum_{l\in \X }f_{jl}(X_l)}_n^2 + \lambda\sum_{l\in \X}\norm{f_{jl}(X_l)}_n
\right\}\\
\widehat S^{(1)}_j &= \left\{
	k: \hat f^{(r)}_{jk} \equiv 0: \hat f^{(r)}_{jk} =\argmin_{\{f_{jl}\}_{l\in \X_{\widehat S^{(0)}}\cup\Y_{-j}}\in \mathcal{F}^{(r)}} 
\norm{Y_j - \sum_{l\in \X_{\widehat S^{(0)}}\cup\Y_{-j} }f_{jl}(W_l)}_n^2 \right.\\
&\quad\quad+\left.\lambda\sum_{l\in \X_{\widehat S^{(0)}}\cup\Y_{-j}}\norm{f_{jl}(W_l)}_n,
\right\}
\end{align*}
where the group lasso penalties are used to jointly penalize the $r$ smoothing bases of each variable. 
This estimator is similar to SPACEJAM \citep{voorman_graph_2014} with edges disagreeing with the layering information hard-coded as zero. 
The following results relies on the general error computation in \citet{haris_generalized_2022}.
A similar result can be found in \cite{tan_doubly_2019}. 

We next state the smoothness assumption needed on the basis approximation. 
\begin{assm}[Truncated basis]\label{ass:gam_truncation}
Suppose there exists $r=O(1)$ such that $\{f_{uv}\}$ are sufficiently sooth in the sense that
\[
|f_{uv}^{(r)}(x_u) - f_{uv}(x_u) |=O_p(1/r^t)
\]
uniformly for all $u,v\in V$ for some $t\in \mathbb{N}$.
\end{assm}

Next we assume two standard conditions for generalized additive models (GAMs). The compatibility condition may be shown for random design, but for simplicity we just assume the conditions. 

\begin{assm}[Compatibility]\label{ass:gam_compat}
Suppose there exists some compatibility constant $\phi>0$ such that if for all $j\in \Y$ and all functions in the form of $f^{(r)}_j(x)=\sum_{k\in \X}f_{jk}^{(r)}(x_k)$ that satisfy $\sum_{k\in \X\setminus{S^{(0)}_j}}\norm{f_{jk}^{(r)}}_n\leq 3\sum_{k\in S^{(0)}_j}\norm{f_k^{(r)}}_n$, it holds that
\[
\sum_{k\in S^{(0)}_j}\norm{f_{jk}^{(r)}}_n\leq \norm{f_{k}^{(r)}}\sqrt{|S^{(0)}_j|}/\phi,\] for some norm $\norm{\cdot}$. 

Also assume that for all $j\in \Y$ and all functions in the form of $f^{(r)}_j(x)=\sum_{k\in \X\cup \Y_{-j}}f_{jk}^{(r)}(w_k)$ that satisfy $\sum_{k\in \X\cup\Y_{-j}\setminus{S^{(1)}_j}}\norm{f_{jk}^{(r)}}_n\leq 3\sum_{k\in S^{(1)}_j}\norm{f_k^{(r)}}_n$, 
\[
\sum_{k\in S^{(1)}_j}\norm{f_{jk}^{(r)}}_n\leq \norm{f_{k}^{(r)}}\sqrt{|S^{(1)}_j|}/\phi.\]
\end{assm}
\begin{assm}[GAM screening]\label{ass:gam_screen}
Denote  
    $s_{\max}=\max_{j\in \Y}\max_{i\in\{0,1\}}
    |S^{(i)}_j|$ and suppose 
    $s_{\max}=o(n/\log (p+q))$.
    Let $\lambda\asymp \sqrt{\log (p+q)/n}$. 
    Suppose 
    \[
    \min_{u\in V}\min_{i\in \{0,1\}}\min_{v\in S^{(i)}_j}\norm{f_{vu}^{(r),0}}=\Omega\left(s_{\max} \frac{\log (p+q)}{n}\right).
    \]
    Let $f^*_u$ be an arbitrary function such that $\sum_{i=1}^n f^*_u(x_{u,i})=0$. 
Suppose there exists some constant $M^*$ satisfying
$M^* =O( |S^{(1)}_j|\lambda/\phi^2)$
such that $f_u\in \mathcal{F}^{(r)}_{\text{local}}$ if and only if $\norm{f_u-f^*_u}_n\leq M^*$.
Further suppose that $\boldsymbol\varepsilon (f_u^*)=O(s_{\max}\lambda^2/\phi^2)$ for all $u,k$. The next result follows directly from the results in \cite{haris_generalized_2022}.
\end{assm}
\begin{prop}[Screening with GAM]\label{thm:screen_gam}
    Suppose Assumption~\ref{ass:gam_truncation}-\ref{ass:gam_screen} hold.
    Then with 
    $n\to\infty$, $p+q=O(n^\xi)$ for some $\xi\geq 0$, and  
	the penalty level stated in Assumption~\ref{ass:gam_screen}, 
	the resulting undirected graph from the screening loop of  Algorithm~\ref{alg:framework} 
	is a supergraph of $H$.
\end{prop}

Propositions~\ref{thm:screen_gam} and \ref{thm:cdcov} establish the consistency of the proposed framework in Algorithm~\ref{alg:framework} for learning {\DAGS} from CAMs, facilitating causal structure learning for a expressive family of distributions.

\end{document}